\newif\ifpreprint
\newif\ifdefertext

\preprinttrue 
\defertextfalse

\ifpreprint
\documentclass[letterpaper]{article} 
\usepackage[preprint]{aaai2027}
\else
\documentclass[letterpaper]{article} 
\usepackage[submission]{aaai2027}
\fi  

\usepackage{environ}
\usepackage{etoolbox}

\newcommand{\deferrednotes}{}

\NewEnviron{defertext}[2]{%
  \ifdefertext
    \BODY
  \else
    #1%
    \gappto\deferrednotes{\subsection*{#2}}%
    \expandafter\gappto\expandafter\deferrednotes\expandafter{\BODY}%
  \fi
}

\usepackage{times}  
\usepackage{helvet}  
\usepackage{courier}  
\usepackage[hyphens]{url}  
\usepackage{graphicx} 
\usepackage{natbib}  
\usepackage{caption} 
\usepackage{algorithm}
\usepackage{algorithmic}
\usepackage{amsmath}
\usepackage{booktabs}

\usepackage{graphicx}
\usepackage[disable]{todonotes} 
\usepackage{amssymb}
\usepackage{dsfont} 
\newcommand{\mathbbm}[1]{\mathds{#1}} 
\usepackage{amsthm}
\usetikzlibrary{positioning,arrows.meta,calc}
\allowdisplaybreaks    
\theoremstyle{plain}
\newtheorem{theorem}{Theorem}
\newtheorem{proposition}{Proposition}

\newtheorem{lemma}{Lemma}  
\theoremstyle{definition}
\newtheorem{definition}{Definition}
\theoremstyle{remark}
\newtheorem{remark}{Remark}
\DeclareMathOperator*{\argmax}{arg\,max}

\newcommand{\Real}{\mathbb{R}}
\newcommand{\tuple}[1]{\langle #1 \rangle}

\usepackage{newfloat}
\usepackage{listings}
\DeclareCaptionStyle{ruled}{labelfont=normalfont,labelsep=colon,strut=off} 
\floatstyle{ruled}
\newfloat{listing}{tb}{lst}{}
\floatname{listing}{Listing}
\title{Reward Structure Shapes the Interaction Between Episodic Exploration and Neural Memory in Reinforcement Learning}

\author{Jai Malegaonkar\equalcontrib\corresponding, Rohan Patil\equalcontrib\corresponding, Henrik I. Christensen}

\affiliations{
Department of Computer Science and Engineering, UC San Diego, United States\\
\{jmalegaonkar, rpatil, hichristensen\}@ucsd.edu
}
\date{\today}

\begin{document}

\maketitle

\begin{abstract}
In partially observable reinforcement learning, agents face a dual bottleneck: they must explore to encounter rewarding states and retain that experience in memory to optimize their policies. Exploration bonuses and memory architectures are traditionally evaluated in isolation, leaving their interaction unmeasured, and standard notions of sparse reward conflate temporal signal density with what the reward actually supervises. We present a controlled study crossing episodic exploration bonuses with diverse neural memory architectures across three environments that vary how the content of memory is acquired. An identical bonus signal yields three distinct interaction patterns: it \emph{amplifies} architectural capacity differences where memory content must be actively discovered and retained unsupervised; \emph{equalizes} architectures to a shared ceiling where the content, once sought out, is a single reward-supervised cue; and is \emph{null} where the observation stream is purely scheduled. Controlled reward manipulations verify that these patterns track reward \emph{structure} rather than density: a dense reward neutralizes a bonus only if it directly supervises the required latent memory, and a small avoidable penalty on exploratory actions (leaving the optimum unchanged) induces policy convergence to suboptimal stationary states, which either bonus resolves. We then formalize reward sparsity with observation-anchored reward machines, separating \emph{structural} sparsity (an automaton reproduces the return without the task-required history) from \emph{potential} sparsity (the one-step reward misprices local exploratory actions); the resulting vocabulary organizes the three regimes by the retention burden each task exposes. Together, these results show exploration and memory are complements, not substitutes: a bonus induces exposure, and only memory converts exposure into return.
\end{abstract}

\section{Introduction}
\label{sec:intro}

An agent acting under partial observability cannot condition on the true state of the world. It must instead act on a summary of its interaction history $(o_0, a_0, \ldots, o_t)$~\citep{kaelbling1998planning}, which deep reinforcement learning approximates with a recurrent hidden state. Two conditions are necessary for this representation to be useful, and each depends on the other: a policy cannot reliably reach rewarding states without an adequate memory representation, and the recurrent state cannot learn to retain critical information without a reachable reward signal to provide gradients.

Exploration and memory are nonetheless studied in isolation. Memory benchmarks isolate recurrent capacity by holding the training signal fixed~\citep{morad2023popgym, pleines2023memorygym, smirnov2025rlbenchnet}, whereas exploration studies evaluate bonuses under a fixed architecture~\citep{henaff2022e3b, zhang2021noveld, yuan2024rlexplore}. Two questions therefore remain open: does the effect of an episodic bonus depend on the \emph{choice} of memory architecture, and what property of the task governs that dependence? Standard terminology provides limited insight into the second question, as the conventional notion of ``sparse reward'' measures signal frequency, conflating temporal density with what the reward actually supervises. We address both questions by crossing episodic bonuses with diverse memory architectures under controlled reward manipulations, and by formalizing reward sparsity in terms of what the reward supervises rather than how often it pays.

We contribute the following:
\begin{itemize}
    \item \textbf{A comprehensive evaluation of episodic bonuses across diverse memory architectures, yielding three distinct interaction patterns.} Holding the training stack constant, an identical episodic bonus \emph{amplifies} the performance spread between architectures where memory content must be actively discovered; \emph{equalizes} architectures to a shared performance ceiling where critical information is initially observable but temporally distant; and yields a \emph{null} effect where the observation stream is purely scheduled.

    \item \textbf{Causal verification that reward \emph{structure}, not reward \emph{density}, governs bonus efficacy.} Holding the environment, architectures, and hyperparameters fixed, we vary solely the target of the reward signal. A dense reward that directly supervises the required memory renders an exploration bonus redundant and often detrimental, whereas a density-matched reward evaluating task-irrelevant transitions preserves its positive effect. A small, avoidable penalty on exploratory actions---which leaves the theoretical optimum unchanged---further drives every architecture into suboptimal stationary policies, which either bonus escapes independently of its magnitude relative to the penalty.

    \item \textbf{A reward-machine formalization of reward sparsity.} We equip POMDPs with observation-anchored reward machines~\citep{toroicarte2018reward} to separate two properties the frequency-based notion conflates: \emph{structural} sparsity, where an automaton reproduces the expected return without retaining the task-required history, leaving the latent memory representation unsupervised; and \emph{potential} sparsity, where the one-step reward misprices local exploratory actions. Classifying each task under a reward-preserving abstraction separates the three empirical patterns by the retention burden each exposes.
\end{itemize}

Adjacent work has examined parts of this space: intrinsic rewards underperforming with LSTMs~\citep{yuan2024rlexplore}, exploration at fixed architecture~\citep{impact2025}, memory decoupled from credit assignment~\citep{ni2023decoupling}, and intrinsic motivation unified at the policy level~\citep{lidayan2025bamdp}. However, none evaluates architecture, bonus, and reward structure simultaneously. Exploration and memory are complementary, not interchangeable: state coverage is ineffective without the representational capacity to retain it, and high-capacity memory requires sufficient coverage to provide a meaningful training signal.

The paper proceeds as follows. We review the partially observable setting alongside foundational concepts of exploration bonuses, reward shaping, and reward machines. We then detail the experimental design across architectures, bonuses, environments, and reward variants. We present the three interaction patterns and demonstrate through controlled manipulations that they track reward structure rather than density.We subsequently formalize reward sparsity as an ordered diagnostic and conclude with a discussion of the results and the limitations of the proposed framework.
\section{Background}
\label{sec:preliminaries}

\subsection{POMDPs and Memory in RL}
\label{subsec:memory}
A finite partially observable Markov decision process (POMDP) is a tuple $\mathcal{M}=\langle S,A,O,p,\omega,r,\gamma,\mu\rangle$: finite state, action, and observation sets, transition kernel $p(s'\mid s,a)$, observation distribution $\omega(o\mid s')$, deterministic reward over traces $r:(S\times A)^*\times S\to\mathbb{R}$, discount $\gamma\in[0,1)$, and initial distribution $\mu$. An agent never observes $s_t$; it receives $o_t\sim\omega(\cdot\mid s_t)$ and acts on the history $h_t=(o_0,a_0,\ldots,o_t)$. An optimal policy is therefore a function of a sufficient statistic of $h_t$~\citep{kaelbling1998planning}.

Deep RL approximates this statistic with learned recurrent architectures, ranging from LSTMs and GRUs~\citep{hausknecht2015deep, kapturowski2019recurrent} to state-space and linear-recurrent models~\citep{gu2023mamba, lu2023structured}. We reserve \emph{memory} for this recurrent state, distinct from episodic replay buffers~\citep{badia2020never}. Prior benchmarks isolate memory capacity by holding the training signal fixed~\citep{morad2023popgym, pleines2023memorygym, smirnov2025rlbenchnet}, leaving the effect of an episodic bonus on the performance gaps between architectures unmeasured.

\subsection{Exploration Bonuses}
\label{subsec:episodic_bonuses}
Intrinsic bonuses augment extrinsic rewards with a novelty signal. \emph{Global} bonuses score a state against lifelong experience~\citep{bellemare2016unifying, pathak2017curiosity, burda2019exploration}; \emph{episodic} bonuses pay within-trajectory novelty and reset every episode~\citep{savinov2019episodic}, driving performance in procedurally generated environments~\citep{henaff2022e3b, henaff2023study}. We study two episodic bonuses alongside a no-bonus control.

E3B pays the elliptical novelty of the current observation from the regularized covariance of features $\phi$ trained by an inverse-dynamics model:
\begin{equation}
\label{eq:e3b}
    \begin{aligned}
        b^{\mathrm{e3b}}_t
        &= \phi(o_t)^{\top}\,C_{t-1}^{-1}\,\phi(o_t), \\
        C_t &= \sum_{i\le t} \phi(o_i)\,\phi(o_i)^{\top} + \lambda I,
    \end{aligned}
\end{equation}
with $C$ reset at episode boundaries~\citep{henaff2022e3b}. NovelD pays the positive difference in lifelong RND novelty $w(\cdot)$ across a transition, gated by an episodic first-visit indicator $N_{\mathrm{ep}}$~\citep{zhang2021noveld}:
\begin{equation}
\label{eq:noveld}
  b^{\mathrm{nvd}}_t
  = \max\!\big(w(o_{t+1})-\alpha\,w(o_t),\,0\big)\cdot\mathbbm{1}\big[N_{\mathrm{ep}}(o_{t+1})=1\big].
\end{equation}
In discrete navigation tasks, E3B's bonus approximates an episodic position count. The two methods therefore differ primarily in their temporal credit assignment: E3B distributes a graded score at every step, whereas NovelD's first-visit gate concentrates payment on a sparse subset of transitions. Both evaluate novelty over raw observations $o_t$, whereas the governing quantity under partial observability is a statistic of history, which the recurrent state approximates. Consequently, the exploration bonus and the memory architecture are structurally coupled.

Because episodic bonuses functionally modify the reward, the theory of policy invariance applies. \citet{ng1999policy} established that a shaping term preserves optimal behavior exactly when it is potential-based: $\gamma\Phi(s')-\Phi(s)$ for some potential $\Phi$ over states. \citet{wiewiora2003equivalence} related such shaping to an equivalent initialization of the value function. Neither E3B nor NovelD is potential-based, thus lacking invariance guarantees; these bonuses may alter the optimal policy. Our aligned reward variant tests this dynamic directly, and potential functions provide the formal basis for our definition of potential sparsity.

\subsection{Reward Machines}
\label{subsec:reward_machines}
A reward machine (RM) summarizes the reward-relevant history of a non-Markovian reward in a finite-state automaton over high-level propositions~\citep{toroicarte2018reward, toroicarte2022reward}. An RM is a tuple $\mathcal{R}=\langle U,u_0,\delta_u,\delta_r \rangle$: a finite state set $U$, an initial state $u_0\in U$, a state-transition function $\delta_u:U\times 2^{\mathcal{P}}\to U$ over a finite proposition set $\mathcal{P}$, and a reward-transition function $\delta_r:U\times 2^{\mathcal{P}}\to\mathbb{R}$.\footnote{We write $\bar{U}:=U\cup F$ for the union of non-terminal and terminal states and use $\bar{U}$ where the distinction is immaterial; \citet{toroicarte2018reward} keep the two strictly separate.} A Markov decision process with a reward machine (MDPRM) is $\mathcal{M'}=\langle S,A,p,\gamma,\mu,P,L,\mathcal{R}_{PSA} \rangle$, equipped with a labelling function $L:S\times A\times S\to 2^{P}$. Subscripts on $\mathcal{R}$ record the labelling domain: $\mathcal{R}_{PSA}$ over propositions, states, and actions; $\mathcal{R}_{POA}$ over propositions, observations, and actions.

The framework extends to POMDPs by anchoring labels to observations, with $L:O\times A\times O\to 2^{\mathcal{P}}$ advancing the machine along a trace. Writing $x_t$ for the RM state at step $t$ ($x_0=u_0$, $x_t=\delta_u(x_{t-1},L(o_{t-1},a_{t-1},o_t))$), an RM is \emph{perfect} for $\mathcal{M}$ w.r.t.\ $L$~\citep{toroicarte2019learning, toroicarte2023learning} iff
\begin{equation}
  \Pr\!\big(o_{t+1}, r_t\mid o_0, a_0, \dots, o_t, a_t\big)
  =
  \Pr\!\big(o_{t+1}, r_t\mid o_t, x_t, a_t\big).
  \label{eq:perfect_rm}
\end{equation}
Augmenting the observation with a perfect RM state renders the environment Markovian over $O\times U$. This transformation strictly isolates tasks where the native reward supervises memory-advancing transitions from those where it does not.
\section{Experimental Setup}
\label{sec:setup}

\paragraph{Common Design.}
Every experiment crosses six architectures (GRU, LSTM, RetNet, GatedDeltaNet, Mamba-2, Memoryless) with a no-bonus control (\emph{none}) and the episodic bonuses E3B (Eq.~\ref{eq:e3b}) and NovelD (Eq.~\ref{eq:noveld}), evaluated as a complete matrix across all three environments. This lineup spans standard gated RNNs (GRU, LSTM), retention and linear-attention models (RetNet, GatedDeltaNet), a state-space model (Mamba-2), and a memoryless control. We run online recurrent PPO with truncated backpropagation through time.

\paragraph{Bonus and memory: coupled, but separate.}
The exploration bonus is strictly unconditioned on the agent's memory (Figure~\ref{fig:coupling}). The bonus encoder $\phi$ is a separate flat MLP that observes only the instantaneous observation $o_t$, isolating it from the recurrent hidden state $h_t$. The bonus signal augments the reward \emph{before} advantage estimation ($r_t + \lambda b_t$), ensuring its gradient backpropagates through time into the recurrent cell alongside the task reward. Consequently, the bonus and memory are coupled during gradient updates but strictly separated in the forward pass. Any observed architecture-dependent effect therefore reflects variations in representational capacity given a \emph{shared} training signal, rather than variations in the signal received. The bonus supplies advantage on transitions where the extrinsic reward is flat; determining whether this alters the underlying latent encoding is beyond the scope of this study.

\begin{figure}[t]
\centering
\includegraphics[width=\columnwidth]{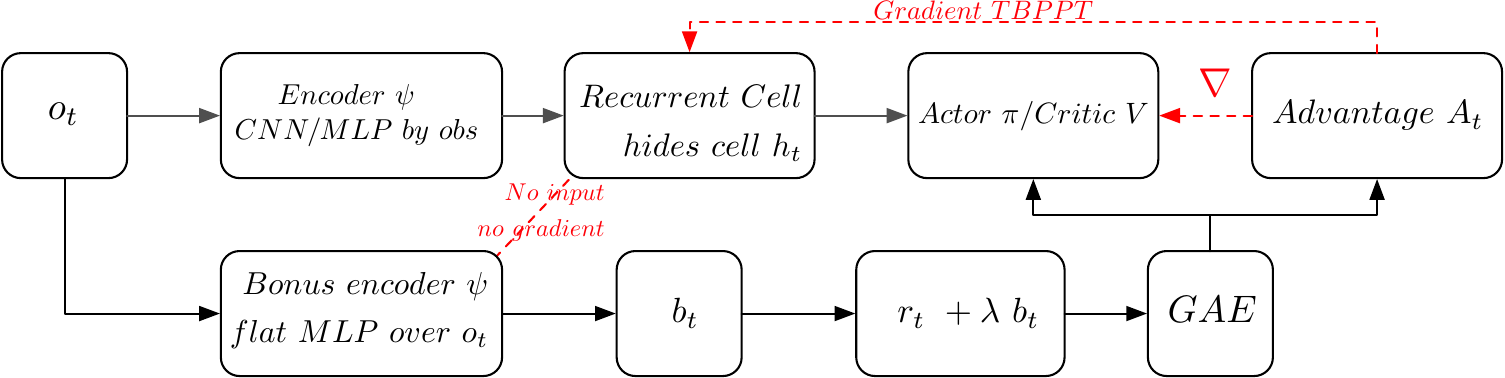}
\caption{Architectural separation of bonus and memory in the forward pass. The bonus encoder $\phi$ processes only the instantaneous observation $o_t$, bypassing the recurrent hidden state $h_t$. Because the bonus augments the reward prior to advantage estimation, its gradient updates the recurrent cell without the memory conditioning on the exploration signal.}
\label{fig:coupling}
\end{figure}

\paragraph{Environments.}
Three environments vary how the content of memory is acquired (full specifications in the appendix):
\begin{itemize}
    \item \textbf{MysteryPath-Grid}~\citep{pleines2023memorygym}: An invisible path must be actively discovered via probing. Off-path steps reset the agent's position but not the path itself; therefore, state coverage yields zero expected return unless the path is retained in memory across resets within the episode.
    \item \textbf{MiniGrid-MemoryS13}~\citep{minigrid}: A cue object near the start must be retained to solve a temporally distant binary T-junction (chance floor $0.50$). We utilize the restrictive $3\times3$ view, which forces the agent to actively seek the cue, as the discriminative setting (the default $7\times7$ view makes the cue trivial to observe).
    \item \textbf{TinyReproduce} (modified POPGym Autoencode~\citep{morad2023popgym}): A 10-symbol sequence is dictated on a fixed, unalterable schedule, and must subsequently be reproduced in reverse. Two exactly return-matched variants differ strictly in timing: \emph{Dense} pays $+1/k$ per correct token, while \emph{Sparse} pays the same total at the end of the episode; both terminate on the first incorrect token.
\end{itemize}
 
\paragraph{Reward Variants.}
On the two gridworlds, we manipulate the target of the reward signal while holding the environment, architectures, and hyperparameters fixed (calibration details in the appendix): \textbf{Sparse} pays $+1$ solely upon success. \textbf{Penalty} introduces an avoidable cost to exploratory moves ($-1/T_{\max}$ per off-path step in MysteryPath; a per-move cost in MemoryS13), leaving the theoretical optimal policy unchanged. \textbf{Aligned} (MysteryPath only) densely pays $+0.1$ for memory-advancing forward progress, serving as a negative control for the bonus. \textbf{Distractor} (both gridworlds) is density-matched to the aligned reward but fundamentally uninformative, paying $\varepsilon = 0.1/T_{\max}$ solely for re-entering known tiles; the discounted criterion $\varepsilon/(1-\gamma) < 1$ leaves the optimal policy unchanged.

\paragraph{Evaluation \& Budgets.}
The primary metric is deterministic-policy success rate over 20 evaluation episodes per checkpoint, reported as the tail mean over the final 20\% of checkpoints. MysteryPath ($T_{\max}{=}128$, $\gamma{=}0.995$) and MemoryS13 ($T_{\max}{=}845$, $\gamma{=}0.999$) run for 20M steps; TinyReproduce ($\gamma{=}0.99$) runs for 10M. Every cell is evaluated across $n{=}10$ seeds, reported as mean $\pm$ std, with all significance claims supported by 95\% unpaired seed-bootstrap intervals; learning curves with seed bands are in the appendix.~\citep{agarwal2021deep}. A single hyperparameter configuration is shared across all architectures (environment-specific overrides are applied identically to all cells), and all experimental runs utilize a bonus coefficient of $\lambda{=}0.03$ derived from a unified sweep. Full registry and calibration details are provided in the appendix.

\begin{table*}[!t]\centering
\caption{MysteryPath-Grid: deterministic-policy success rate. Values denote the tail mean over the final 20\% of evaluation checkpoints (mean $\pm$ std, $n{=}10$ seeds). \textbf{Bold} text denotes frozen cells ($.00$ success rate). E3B and NovelD bonuses are crossed with diverse memory architectures across four controlled reward manipulations.}
\label{tab:mpg}\small\setlength{\tabcolsep}{3pt}
\begin{tabular}{@{}l ccc ccc ccc ccc@{}}\toprule
 & \multicolumn{3}{c}{Sparse} & \multicolumn{3}{c}{Penalty} & \multicolumn{3}{c}{Aligned} & \multicolumn{3}{c}{Distractor} \\
\cmidrule(lr){2-4}\cmidrule(lr){5-7}\cmidrule(lr){8-10}\cmidrule(lr){11-13}
Cell & none & e3b & nvld & none & e3b & nvld & none & e3b & nvld & none & e3b & nvld \\\midrule
GRU & .14{\scriptsize$\pm$.01} & .18{\scriptsize$\pm$.02} & .18{\scriptsize$\pm$.02} & \textbf{.00}{\scriptsize$\pm$.00} & .17{\scriptsize$\pm$.01} & .13{\scriptsize$\pm$.01} & .31{\scriptsize$\pm$.03} & .18{\scriptsize$\pm$.02} & .26{\scriptsize$\pm$.02} & .03{\scriptsize$\pm$.05} & .18{\scriptsize$\pm$.02} & .18{\scriptsize$\pm$.01} \\
LSTM & .14{\scriptsize$\pm$.01} & .17{\scriptsize$\pm$.02} & .19{\scriptsize$\pm$.02} & \textbf{.00}{\scriptsize$\pm$.00} & .17{\scriptsize$\pm$.02} & .13{\scriptsize$\pm$.01} & .26{\scriptsize$\pm$.02} & .18{\scriptsize$\pm$.02} & .26{\scriptsize$\pm$.02} & .06{\scriptsize$\pm$.06} & .18{\scriptsize$\pm$.02} & .19{\scriptsize$\pm$.02} \\
RetNet & .16{\scriptsize$\pm$.03} & .56{\scriptsize$\pm$.05} & .53{\scriptsize$\pm$.03} & \textbf{.00}{\scriptsize$\pm$.00} & .55{\scriptsize$\pm$.02} & .57{\scriptsize$\pm$.04} & .75{\scriptsize$\pm$.05} & .61{\scriptsize$\pm$.05} & .78{\scriptsize$\pm$.04} & .07{\scriptsize$\pm$.06} & .55{\scriptsize$\pm$.03} & .53{\scriptsize$\pm$.05} \\
GatedDeltaNet & .17{\scriptsize$\pm$.02} & .63{\scriptsize$\pm$.04} & .59{\scriptsize$\pm$.06} & \textbf{.00}{\scriptsize$\pm$.00} & .58{\scriptsize$\pm$.07} & .59{\scriptsize$\pm$.06} & .67{\scriptsize$\pm$.10} & .65{\scriptsize$\pm$.08} & .82{\scriptsize$\pm$.04} & .02{\scriptsize$\pm$.06} & .62{\scriptsize$\pm$.05} & .57{\scriptsize$\pm$.05} \\
Mamba-2 & .20{\scriptsize$\pm$.02} & .34{\scriptsize$\pm$.04} & .30{\scriptsize$\pm$.03} & \textbf{.00}{\scriptsize$\pm$.00} & .31{\scriptsize$\pm$.04} & .25{\scriptsize$\pm$.03} & .44{\scriptsize$\pm$.03} & .31{\scriptsize$\pm$.03} & .44{\scriptsize$\pm$.05} & .02{\scriptsize$\pm$.05} & .34{\scriptsize$\pm$.05} & .29{\scriptsize$\pm$.03} \\
Memoryless & .06{\scriptsize$\pm$.01} & .00{\scriptsize$\pm$.00} & .04{\scriptsize$\pm$.01} & \textbf{.00}{\scriptsize$\pm$.00} & \textbf{.00}{\scriptsize$\pm$.00} & .01{\scriptsize$\pm$.00} & .05{\scriptsize$\pm$.01} & .01{\scriptsize$\pm$.01} & .05{\scriptsize$\pm$.01} & .01{\scriptsize$\pm$.02} & .00{\scriptsize$\pm$.00} & .04{\scriptsize$\pm$.01} \\
\bottomrule\end{tabular}\end{table*}

\section{Results}
\label{sec:experiments}

To evaluate the interaction between representational capacity and exploration, we crossed diverse memory architectures with episodic exploration bonuses across three distinct POMDP environments. We find that exploration bonuses do not affect memory architectures uniformly; rather, the interaction yields three distinct empirical patterns tracking the underlying reward structure of the environment. Furthermore, controlled manipulations of the reward structure verify that these patterns are governed by what the reward supervises, not merely its temporal density. Gains reported below represent within-cell changes in mean success against the respective no-bonus control, supported by 95\% unpaired seed-bootstrap intervals.

\begin{table*}[t]\centering
\caption{MiniGrid-MemoryS13: deterministic-policy success rate. Values denote the tail mean over the final 20\% of evaluation checkpoints (mean $\pm$ std, $n{=}10$ seeds). E3B and NovelD bonuses are crossed with memory architectures across restricted ($3\times3$) and default ($7\times7$) fields of view, alongside penalty and distractor reward manipulations.}
\label{tab:s13}\small\setlength{\tabcolsep}{3pt}
\begin{tabular}{@{}l ccc ccc ccc ccc@{}}\toprule
 & \multicolumn{3}{c}{$3\times3$ Sparse} & \multicolumn{3}{c}{$3\times3$ Penalty} & \multicolumn{3}{c}{$3\times3$ Distractor} & \multicolumn{3}{c}{$7\times7$ Sparse} \\
\cmidrule(lr){2-4}\cmidrule(lr){5-7}\cmidrule(lr){8-10}\cmidrule(lr){11-13}
Cell & none & e3b & nvld & none & e3b & nvld & none & e3b & nvld & none & e3b & nvld \\\midrule
GRU & .50{\scriptsize$\pm$.01} & .92{\scriptsize$\pm$.09} & .86{\scriptsize$\pm$.19} & .50{\scriptsize$\pm$.01} & .95{\scriptsize$\pm$.05} & .94{\scriptsize$\pm$.10} & .50{\scriptsize$\pm$.02} & .87{\scriptsize$\pm$.09} & .96{\scriptsize$\pm$.04} & .69{\scriptsize$\pm$.19} & .96{\scriptsize$\pm$.10} & .97{\scriptsize$\pm$.08} \\
LSTM & .50{\scriptsize$\pm$.02} & .98{\scriptsize$\pm$.02} & .94{\scriptsize$\pm$.10} & .50{\scriptsize$\pm$.01} & .94{\scriptsize$\pm$.06} & .98{\scriptsize$\pm$.02} & .50{\scriptsize$\pm$.01} & .99{\scriptsize$\pm$.02} & .99{\scriptsize$\pm$.02} & .76{\scriptsize$\pm$.22} & .96{\scriptsize$\pm$.07} & .98{\scriptsize$\pm$.03} \\
RetNet & .49{\scriptsize$\pm$.02} & .91{\scriptsize$\pm$.11} & .79{\scriptsize$\pm$.19} & .49{\scriptsize$\pm$.01} & .86{\scriptsize$\pm$.12} & .86{\scriptsize$\pm$.11} & .50{\scriptsize$\pm$.01} & .86{\scriptsize$\pm$.15} & .73{\scriptsize$\pm$.21} & .67{\scriptsize$\pm$.17} & .96{\scriptsize$\pm$.06} & .96{\scriptsize$\pm$.05} \\
GatedDeltaNet & .76{\scriptsize$\pm$.17} & .97{\scriptsize$\pm$.04} & .98{\scriptsize$\pm$.03} & .68{\scriptsize$\pm$.17} & .96{\scriptsize$\pm$.07} & .97{\scriptsize$\pm$.05} & .73{\scriptsize$\pm$.22} & .98{\scriptsize$\pm$.03} & .99{\scriptsize$\pm$.02} & .94{\scriptsize$\pm$.12} & .93{\scriptsize$\pm$.15} & .99{\scriptsize$\pm$.02} \\
Mamba-2 & .86{\scriptsize$\pm$.14} & .76{\scriptsize$\pm$.23} & .91{\scriptsize$\pm$.06} & .79{\scriptsize$\pm$.24} & .84{\scriptsize$\pm$.19} & .85{\scriptsize$\pm$.12} & .74{\scriptsize$\pm$.21} & .73{\scriptsize$\pm$.31} & .81{\scriptsize$\pm$.21} & .87{\scriptsize$\pm$.15} & .93{\scriptsize$\pm$.11} & .96{\scriptsize$\pm$.03} \\
Memoryless & .35{\scriptsize$\pm$.24} & .50{\scriptsize$\pm$.01} & .50{\scriptsize$\pm$.01} & .50{\scriptsize$\pm$.01} & .48{\scriptsize$\pm$.02} & .50{\scriptsize$\pm$.01} & .50{\scriptsize$\pm$.01} & .49{\scriptsize$\pm$.01} & .50{\scriptsize$\pm$.01} & .50{\scriptsize$\pm$.01} & .48{\scriptsize$\pm$.02} & .50{\scriptsize$\pm$.01} \\
\bottomrule\end{tabular}\end{table*}

\begin{table*}[t]\centering
\caption{TinyReproduce ($k{=}10$, reverse): exact-recall success rate. Values denote the tail mean over the final 20\% of evaluation checkpoints (mean $\pm$ std, $n{=}10$ seeds). E3B and NovelD bonuses are crossed with memory architectures across sparse and dense reward variants.}
\label{tab:tiny}\small\setlength{\tabcolsep}{6pt}
\begin{tabular}{@{}l ccc ccc@{}}\toprule
 & \multicolumn{3}{c}{Sparse} & \multicolumn{3}{c}{Dense} \\
\cmidrule(lr){2-4}\cmidrule(lr){5-7}
Cell & none & e3b & nvld & none & e3b & nvld \\\midrule
GRU & .90{\scriptsize$\pm$.02} & .90{\scriptsize$\pm$.02} & .90{\scriptsize$\pm$.01} & .91{\scriptsize$\pm$.02} & .90{\scriptsize$\pm$.02} & .90{\scriptsize$\pm$.02} \\
LSTM & .91{\scriptsize$\pm$.02} & .93{\scriptsize$\pm$.02} & .92{\scriptsize$\pm$.02} & .91{\scriptsize$\pm$.02} & .92{\scriptsize$\pm$.01} & .91{\scriptsize$\pm$.03} \\
RetNet & .00{\scriptsize$\pm$.00} & .00{\scriptsize$\pm$.00} & .00{\scriptsize$\pm$.00} & .00{\scriptsize$\pm$.00} & .00{\scriptsize$\pm$.00} & .00{\scriptsize$\pm$.00} \\
GatedDeltaNet & .33{\scriptsize$\pm$.08} & .31{\scriptsize$\pm$.10} & .40{\scriptsize$\pm$.15} & .35{\scriptsize$\pm$.09} & .32{\scriptsize$\pm$.11} & .39{\scriptsize$\pm$.07} \\
Mamba-2 & .01{\scriptsize$\pm$.01} & .01{\scriptsize$\pm$.01} & .01{\scriptsize$\pm$.01} & .01{\scriptsize$\pm$.01} & .01{\scriptsize$\pm$.01} & .01{\scriptsize$\pm$.01} \\
Memoryless & .00{\scriptsize$\pm$.00} & .00{\scriptsize$\pm$.00} & .00{\scriptsize$\pm$.00} & .00{\scriptsize$\pm$.00} & .00{\scriptsize$\pm$.00} & .00{\scriptsize$\pm$.00} \\
\bottomrule\end{tabular}\end{table*}

\subsection{Exploration Does Not Help Memory Uniformly}
\label{subsec:three-patterns}

On MysteryPath-Grid, all memory architectures exhibit poor baseline performance without an exploration bonus ($0.14$--$0.20$; the memoryless control sits at $0.06$). The introduction of E3B does not provide a uniform performance increase (Table~\ref{tab:mpg}). Instead, the within-cell gains selectively order the architectures: standard gated RNNs exhibit negligible improvements (LSTM $+0.03$ $[+0.02,+0.05]$, GRU $+0.04$ $[+0.03,+0.06]$), whereas linear-attention and retention models demonstrate substantial performance increases (RetNet $+0.40$ $[+0.37,+0.44]$, GatedDeltaNet $+0.46$ $[+0.44,+0.49]$), with an intermediate gain for the state-space model Mamba-2 ($+0.14$). These three cells maintain disjoint seed distributions with and without the bonus, whereas GRU and LSTM distributions overlap heavily. We designate this the \emph{amplification} pattern: the exploration signal disproportionately amplifies the underlying sequence-modeling capacity of the architecture. A coefficient sweep confirms this is a structural phenomenon rather than a hyperparameter artifact: across an order of magnitude of $\lambda$, GRU's performance remains low and fails to scale ($0.175 \to 0.133$) while RetNet's scales monotonically ($0.300 \to 0.475$). Conversely, the memoryless control actively declines under E3B ($0.062 \to 0.002$) with disjoint distributions, demonstrating empirically that state coverage without recurrent representational capacity strictly degrades performance.

On MiniGrid-MemoryS13 ($3\times3$), critical information is initially observable and the architectural responses invert (Table~\ref{tab:s13}). Without a bonus, GRU, LSTM, and RetNet average $0.49$--$0.50$, making them statistically indistinguishable from chance at a binary junction. Under E3B, every seed in these three cells reaches at least $0.68$, generating significant within-cell improvements ($+0.42$, $+0.48$, and $+0.42$, respectively); NovelD produces comparable gains ($+0.36$, $+0.44$, and $+0.30$). Conversely, models that independently converge above chance without a bonus demonstrate marginal or zero improvement, with Mamba-2's interval spanning zero ($-0.10$ $[-0.26,+0.06]$). We designate this the \emph{equalization} pattern: the bonus dictates whether low-capacity cells learn the task at all, effectively pulling them to the performance ceiling rather than further benefiting already-competent models.

On TinyReproduce, environment dynamics are purely scheduled and no sequence information can be actively sought. Consequently, no bonus produces a substantive within-cell effect (Table~\ref{tab:tiny}). GRU and LSTM converge on the task independently and remain unaffected; GatedDeltaNet plateaus at $0.30$--$0.40$ regardless of the applied bonus. RetNet, Mamba-2, and the Memoryless control fail to learn exact reverse recall, aligning with known copying gaps documented for state-space and linear-attention models~\citep{jelassi2024repeat}. All measured bonus effects lie strictly within $\pm0.08$. The single interval excluding zero is LSTM under E3B on the sparse variant ($+0.018$ $[+0.001,+0.034]$), representing a statistically negligible fraction of its baseline performance. This constitutes the \emph{null} pattern.

\subsection{Reward Structure vs. Reward Density}
\label{subsec:not-density}

To isolate whether these interactions are driven by reward density or structural task dependencies, we manipulated what the reward explicitly supervises while holding frequency and hyperparameters fixed.

The \emph{aligned} reward on MysteryPath densely pays $+0.1$ for memory-advancing forward progress. Under this structure, RetNet and GatedDeltaNet reach $0.75$ and $0.67$ natively (Table~\ref{tab:mpg}). However, adding E3B degrades performance for four of the five memory architectures: GRU drops by $-0.13$, LSTM by $-0.08$, RetNet by $-0.14$, and Mamba-2 by $-0.13$ (all four intervals excluding zero), whereas GatedDeltaNet is essentially unchanged ($0.67 \to 0.65$). Because E3B provides a dense novelty signal at every step, it functionally interferes with the dense extrinsic supervision. NovelD, constrained by its first-visit gating mechanism, interferes far less: its effects are neutral for LSTM, RetNet, and Mamba-2, mildly negative for GRU ($0.31 \to 0.26$), and positive for GatedDeltaNet ($0.67 \to 0.82$).

Conversely, the \emph{distractor} reward matches the aligned variant's density but exclusively pays for uninformative state visitations (re-entering known tiles). This useless density induces policy convergence to suboptimal states, collapsing every memory architecture's no-bonus baseline to between $0.02$ and $0.07$. Despite this dense signal, E3B successfully restores RetNet and GatedDeltaNet to their sparse-reward performance levels ($0.55$ and $0.62$, compared to $0.56$ and $0.63$ on sparse). This dissociation shows that, in these tasks, a dense extrinsic reward neutralizes an exploration bonus only when it directly supervises the memory the task requires; matching its density with an uninformative signal confers no such effect.

\subsection{Pricing Exploration Can Hinder Optimization}
\label{subsec:penalty}

The \emph{penalty} reward variant on MysteryPath imposes a $-1/T_{\max}$ cost per off-path step. Because falls are avoidable, the theoretical mathematical optimum is identical to the sparse reward. Despite this equivalence, the penalty induces policy convergence to a $0.00$ success rate across all architectures (Table~\ref{tab:mpg}). Trajectory data confirms a stationary, suboptimal policy: over the final 80\% of training, median falls per episode reach $0.00$, episode lengths sit exactly at $T_{\max}{=}128$, and the critic's loss approaches $0.00$. The policies remain highly stochastic (exhibiting high policy entropy) but exclusively within verified-safe, zero-cost states.

Both E3B and NovelD resolve this exploratory stagnation, restoring RetNet and GatedDeltaNet to their sparse-reward performance levels ($0.55$--$0.60$). This recovery occurs independently of the specific bonus magnitude: E3B delivers $2.6\times$ the per-step penalty, while NovelD delivers only $0.3\times$, yet both recover GatedDeltaNet equally ($0.58$--$0.59$). This optimization failure is strictly dependent on the environmental state space; the matched penalty on MemoryS13 does not cause a collapse (Table~\ref{tab:s13}) because no equivalent absorbing zero-cost state exists.
%
%

\section{Reward Sparsity \& Reward Machines}
\label{sec:sparsity}

Conventional definitions of reward sparsity rely strictly on the frequency of non-zero payments, measuring temporal density rather than structural dependencies. This temporal metric fails to explain our empirical results: the base rewards in MysteryPath, MemoryS13, and TinyReproduce are identically sparse by frequency, yet an identical episodic bonus amplifies, equalizes, or nullifies architecture differences across them. We therefore formalize two distinct notions of sparsity using observation-anchored reward machines. \emph{Structural} sparsity evaluates whether the reward can be reproduced by an automaton that omits the history required to solve the task. \emph{Potential} sparsity evaluates whether the immediate one-step reward optimally ranks local actions. While \citet{toroicarte2023learning} defines perfect reward machines for POMDPs, no corresponding formulation equips a POMDP with an arbitrary (not necessarily perfect) reward machine.

\begin{definition}[POMDPRM]
\label{def:pomdprm}
A POMDP with a simple reward machine (POMDPRM) is a tuple
$\mathcal{T} = \tuple{S, A, O, p, \omega, \gamma, \mu, P, L,
\mathcal{R}_{POA}}$ where $S, A, O, p, \omega, \gamma, \mu$ are as in a
POMDP, $P$ is a finite set of propositional symbols,
$L : O \times A \times O \to 2^{P}$ is a labelling function, and
$\mathcal{R}_{POA} = \tuple{U, u_0, \delta_u, \delta_r}$ is an RM with
finite $U$, $\delta_u : U \times 2^{P} \to U$ and
$\delta_r : U \times 2^{P} \to \Real$. The reward process of
$\mathcal{T}$ is given by $x_0 = u_0$,
$x_{t+1} = \delta_u\big(x_t, L(o_t, a_t, o_{t+1})\big)$, and
$\tilde r_t = \delta_r\big(x_t, L(o_t, a_t, o_{t+1})\big)$.
\end{definition}

\begin{definition}[Encoding]
\label{def:encode}
Given a POMDP $\mathcal{M} = \tuple{S, A, O, p, \omega, r, \gamma,
\mu}$,
\begin{itemize}
    \item MDPRM $\tuple{S, A, p, \gamma, \mu, P, L,
    \mathcal{R}_{PSA}}$, where $\mathcal{R}_{PSA} = \tuple{U, u_0,
    \delta_u, \delta_r}$, \emph{encodes} $\mathcal{M}$ $\iff$ on all
    reachable trajectories under every policy,
    $r_t = \delta_r\big(x_t, L(s_t, a_t, s_{t+1})\big)$ for all $t$,
    where $x_t$ is the RM state at $t$;
    \item POMDPRM $\tuple{S, A, O, p, \omega, \gamma, \mu, P, L,
    \mathcal{R}_{POA}}$, where $\mathcal{R}_{POA} = \tuple{U, u_0,
    \delta_u, \delta_r}$, \emph{encodes} $\mathcal{M}$ $\iff$ for
    every observation-history policy $\pi$,
    $$\mathbb{E}^\pi\Big[\sum_{t\ge 0} \gamma^t r_t\Big] =
    \mathbb{E}^\pi\Big[\sum_{t\ge 0} \gamma^t \tilde r_t\Big].$$
\end{itemize}
\end{definition}

The RM of a POMDPRM is not strictly required to be perfect: the expected machine reward on an observed transition is defined as the conditional expectation of the environment reward given the observation. While encoders are not guaranteed to exist globally (the appendix demonstrates a three-state POMDP requiring an unbounded number of machine states), they exist and are computable provided the observation stream induces only finitely many posteriors over the hidden state.

\begin{definition}[Structural sparsity]
\label{def:structural}
Given a POMDP $\mathcal{M} = \tuple{S, A, O, p, \omega, r, \gamma,
\mu}$, $r$ is \emph{structurally sparse} if no POMDPRM encodes
$\mathcal{M}$, or if there exists a POMDPRM
$\tuple{S, A, O, p, \omega, \gamma, \mu, P, L, \mathcal{R}_{POA}}$
which encodes $\mathcal{M}$ such that $\mathcal{R}_{POA}$ is not
perfect w.r.t.\ $L$ (Eq.~\ref{eq:perfect_rm}). $r$ is \emph{structurally
dense} if at least one POMDPRM encodes $\mathcal{M}$ and, for all
POMDPRMs $\tuple{S, A, O, p, \omega, \gamma, \mu, P, L,
\mathcal{R}_{POA}}$ which encode $\mathcal{M}$, $\mathcal{R}_{POA}$ is
perfect w.r.t.\ $L$.
\end{definition}

Structural density indicates that every automaton consistent with the return inherently encodes the memory required by the task. Conversely, structural sparsity indicates that an automaton can reproduce the return without achieving perfection. MysteryPath and MemoryS13 are both structurally sparse, but via distinct mechanisms. In MysteryPath, the one-state machine (Figure~\ref{fig:rm-axis}a) encodes the return while entirely ignoring the path; the reward signal never supervises the memory required to solve the task. In MemoryS13, the corresponding machine (Figure~\ref{fig:rm-axis}b) tracks the cue and encodes the return, yet fails perfection on raw observations due to visual aliasing (e.g., corridor cells share identical views, rendering next-observation predictions dependent on unobserved positions). Here, dynamics aliasing---rather than the reward structure itself---prevents perfection. Conversely, TinyReproduce is structurally dense in its dictation abstraction: the play observation remains constant following both correct and incorrect tokens, forcing every encoder to carry the fully dictated sequence (Figure~\ref{fig:rm-axis}c). Formal proofs are provided in the appendix, including the analysis of a reward-preserving, structurally dense abstraction of MemoryS13 that strictly removes dynamics aliasing.

The second notion of sparsity applies to arbitrary encoders; perfection is not assumed prior to Theorem~\ref{thm:greedy}. 
\begin{definition}[Greedy RM policy]
\label{def:greedy}
Given a POMDPRM $\tuple{S, A, O, p, \omega, \gamma, \mu, P, L,
\mathcal{R}_{POA}}$ and a fixed full-support reference policy
$\pi_0(a \mid o, u) = 1/|A|$, the \emph{greedy RM policy} $\pi$ is
uniquely given by
$$\pi(\cdot \mid o, u) = \argmax_a \sum_{o' \in O}
\Pr\nolimits_{\pi_0}(o' \mid o, u, a)\, \delta_r(u, \sigma),$$
where $\sigma = L(o, a, o')$ and tie-breaking among maximizing actions is uniformly random.
\end{definition}

Here, $\Pr_{\pi_0}(o' \mid o, u, a)$ denotes the next-observation distribution along trajectories generated by $\pi_0$; it coincides with the canonical kernel of Eq.~\ref{eq:perfect_rm} under perfect machines. Fixing $\pi_0$ strictly defines $\pi$.

\begin{definition}[Potential sparsity]
\label{def:potential}
Suppose POMDPRM $\mathcal{T} = \tuple{S, A, O, p, \omega, \gamma, \mu,
P, L, \mathcal{R}_{POA}}$ encodes the POMDP
$\tuple{S, A, O, p, \omega, r, \gamma, \mu}$, where
$\mathcal{R}_{POA} = \tuple{U, u_0, \delta_u, \delta_r}$. Let
$J(\pi) := \mathbb{E}^\pi\big[\sum_{t\ge 0} \gamma^t r_t\big]$ denote
the expected return in $\mathcal{M}$. For $\phi : U \times O \to \Real$
define the shaped one-step value
$$q^\phi(o, u, a) = \sum_{o' \in O}
\Pr\nolimits_{\pi_0}(o' \mid o, u, a)\hat{r}(u, o, o', a)$$
where $\hat{r}(u, o, o', a)=\delta_r\big(u,
L(o, a, o')\big) + \gamma\,\phi(u', o') - \phi(u, o), u' = \delta_u\big(u, L(o, a, o')\big)$, and let $\pi^\phi = \arg\max_a q^\phi(o, u, a)$. If there exists a $\phi$
such that $J(\pi) < J(\pi^\phi)$ where $\pi$ is greedy RM policy for $\mathcal{T}$, then $r$ is \emph{potentially sparse
w.r.t.\ $\mathcal{T}$}; else $r$ is \emph{potentially dense w.r.t.\
$\mathcal{T}$}.
\end{definition}

The correction term $\gamma\,\phi(u', o') - \phi(u, o)$ maintains the potential-based form of \citet{ng1999policy}, defined over the state-observation pairs $(u, o)$ rather than machine edges. The shaped table serves strictly to select one-step policies, whose true returns $J$ are subsequently evaluated. Because a coarse automaton cannot register a misranking that its state space fails to represent, potential sparsity is exclusively evaluated in settings where the reward is structurally dense and all encoders are perfect.

\begin{theorem}[Dense rewards need no lookahead]
\label{thm:greedy}
For a POMDP $\mathcal{M} = \tuple{S, A, O, p, \omega, r, \gamma, \mu}$,
if $r$ is structurally dense, POMDPRM $\mathcal{T} = \tuple{S, A, O, p,
\omega, \gamma, \mu, P, L, \mathcal{R}_{POA}}$ encodes $\mathcal{M}$
(hence is perfect, by Definition~\ref{def:structural}), and $r$ is
potentially dense w.r.t.\ $\mathcal{T}$, then the greedy RM policy over
$\mathcal{T}$ is optimal for $\mathcal{M}$ among all
observation-history policies.
\end{theorem}

Perfection collapses the POMDP to a finite MDP over the pairs $(o, u)$. In this formulation, shaping with the optimal value function converts the one-step reward into the optimal advantage; thus, potential density implies the unshaped greedy policy is already optimal. Theorem~\ref{thm:greedy} strictly characterizes the induced planning problem rather than the convergence dynamics of actor-critic optimization. Notably, the two reward variants of TinyReproduce diverge on potential sparsity with respect to the sequence machine (appendix). 



\begin{figure*}[t]
\centering
\begin{minipage}[b]{0.195\textwidth}\centering
\includegraphics[width=\linewidth]{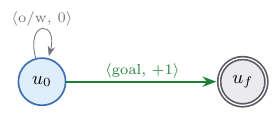}\\[3pt]
{\small (a) MysteryPath}
\end{minipage}\hfill
\begin{minipage}[b]{0.26\textwidth}\centering
\includegraphics[width=\linewidth]{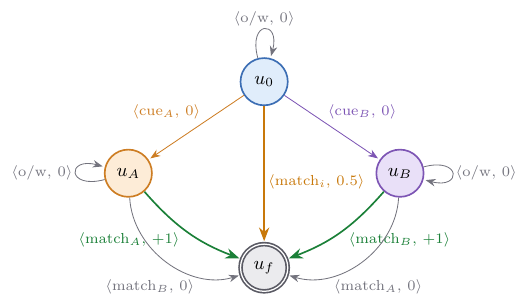}\\[3pt]
{\small (b) MiniGrid-MemoryS13}
\end{minipage}\hfill
\begin{minipage}[b]{0.4\textwidth}\centering
\includegraphics[width=\linewidth]{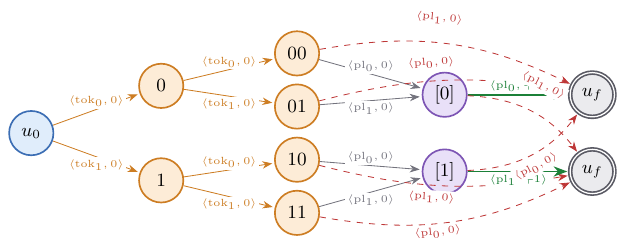}\\[3pt]
{\small (c) TinyReproduce ($k = 2$, reward on full success only)}
\end{minipage}
\caption{\textbf{Reward machines for the three environments.} Transitions are labeled $\langle\varphi,c\rangle$, where $\varphi$ denotes the label condition and $c$ denotes the reward. The symbol $\mathrm{o/w}$ matches any remaining label, and $u_f$ indicates the terminal state. \textbf{(a)}~The MysteryPath machine, featuring a single non-terminal state and issuing $+1$ upon reaching the visible goal. \textbf{(b)}~The MemoryS13 machine, which stores the initially observed cue. The transition $\mathrm{match}_A$ yields $+1$ from $u_A$ and $0$ from $u_B$, while a cue-blind transition earns $0.5$. \textbf{(c)}~The TinyReproduce machine, which stores the dictated sequence. During replay, correctness is verified against stored targets, with $+1$ issued solely upon full sequence completion.}
\label{fig:rm-axis}
\end{figure*}


\begin{defertext}{}{Encoders need not exist in general}
Encoding is demanding even in expectation. The following construction
shows that a Markovian reward over three states can require more
machine states than any finite automaton has.

\begin{proposition}
\label{prop:horizon}
There is a POMDP $\mathcal{M}$ with $|S| = 3$, $|O| = 3$, $|A| = 3$ and
a Markovian reward, such that (i) in unbounded horizon, no finite RM
over observations matches its expected return for all
observation-history policies; and (ii) under any horizon $H$,
\emph{every} RM over observations that does so has at least $H - 1$
states.
\end{proposition}

\begin{proof}
Let $\mathcal{M}$ have states $S = \{s_0, s_1, s_d\}$, actions
$A = \{c, g_0, g_1\}$, and observations $O = \{o^0, o^1, o^\bot\}$.
The start state is $s_0$ or $s_1$ with equal probability. Action $c$
keeps the state fixed, while $g_0$ and $g_1$ move to the absorbing
state $s_d$. In state $s_z$ the observation equals $o^z$ with
probability $q$ and the other symbol with probability $1-q$, for a
fixed $\tfrac12 < q < 1$, and $s_d$ always shows $o^\bot$. The reward
of guess $g_i$ is $1$ if the state was $s_i$ and $0$ otherwise, and
all other rewards are $0$. In plain terms, a hidden coin, noisy peeks,
and a guess that ends the episode. By construction the reward is
Markovian.

For any policy whose first $k-1$ steps choose action $c$,
\begin{align*}
\Pr(s_1 \mid o_0 = \dots = o_{k-1} = o^1)
&= \frac{\frac12 q^k}{\frac12 q^k + \frac12 (1-q)^k}\\
&= \sigma(k\lambda),
\end{align*}
where $\sigma(x) := \frac{e^x}{1+e^x}$ and
$\lambda := \ln\frac{q}{1-q} > 0$. Since
$\sigma'(x) = \sigma(x)(1 - \sigma(x)) > 0$, each additional matching
peek makes the bettor strictly more confident, without bound in $k$.

Suppose a POMDPRM $\mathcal{T}$ with
$\mathcal{R}_{POA} = \tuple{U, u_0, \delta_u, \delta_r}$ encodes
$\mathcal{M}$. Consider the two policies $\pi_k^0$ and $\pi_k^1$ that
take $c$ for the first $k-1$ steps and then take $g_0$ and $g_1$
respectively, and consider the trace on which all observations of the
first $k-1$ steps equal $o^1$. Let $x_0, x_1, x_2, \dots$ be the
machine states along that trace, with $x_0 = u_0$ by definition. For
$k = 1$ the expected returns of $\pi_1^0$ and $\pi_1^1$ are
$1 - \sigma(\lambda)$ and $\sigma(\lambda)$, so matching them forces
$\delta_r\big(x_0, L(o^1, g_0, o^\bot)\big) = 1 - \sigma(\lambda)$ and
$\delta_r\big(x_0, L(o^1, g_1, o^\bot)\big) = \sigma(\lambda)$.

We claim that $x_0, \dots, x_{k-1}$ are pairwise distinct with
$\delta_r\big(x_{k-1}, L(o^1, g_i, o^\bot)\big)$ equal to
$1 - \sigma(k\lambda)$ for $i = 0$ and $\sigma(k\lambda)$ for $i = 1$,
by induction on $k$. Suppose it holds up to $k$. First,
$\delta_r\big(x_j, L(o^1, c, o^1)\big) = 0$ for $0 \le j \le k-1$, as
otherwise the expected returns of $\pi_j^0, \pi_j^1$ would not match.
Now $x_k = \delta_u\big(x_{k-1}, L(o^1, c, o^1)\big)$. If $x_k = x_j$
for some $j \le k-1$, the guess rewards at $x_k$ and $x_j$ coincide;
but the returns of $\pi_{k+1}^0, \pi_{k+1}^1$ pin the former to
$1 - \sigma\big((k{+}1)\lambda\big), \sigma\big((k{+}1)\lambda\big)$
while those of $\pi_{j+1}^0, \pi_{j+1}^1$ pin the latter to
$1 - \sigma\big((j{+}1)\lambda\big), \sigma\big((j{+}1)\lambda\big)$,
and $\sigma$ is strictly increasing, a contradiction.

For any $k + 1 > |U|$ the trace exhibits $k$ distinct machine states,
so by the pigeonhole principle no finite RM encodes the reward at
unbounded horizon. If the maximum horizon is $H$ and the RM encodes
the reward, the states $x_0, \dots, x_{H-2}$ are pairwise distinct, so
$|U| \ge H - 1$.
\end{proof}
\end{defertext}

\begin{defertext}{}{The belief filter is a perfect encoder}
The impossibility above disappears when the observations leave only
finitely many posteriors over the hidden pair. In that case an
encoding machine exists, is computable, and is perfect.

\begin{proposition}[Belief filter]
\label{prop:filter}
Let $\mathcal{M}$ be a POMDP with $\gamma \in (0,1)$ such that some
MDPRM $\tuple{S, A, p, \gamma, \mu, P', L', \mathcal{R}_{P'SA}}$, with
$\mathcal{R}_{P'SA} = \tuple{U', u_0', \delta_u', \delta_r'}$, encodes
$\mathcal{M}$, and suppose the set of posteriors over
$Z := S \times U'$ reachable at positive probability is finite. Then
the belief filter over $Z$, whose reward on each transition is the
average of the latent reward given the observable past, is an
explicitly computable POMDPRM that encodes $\mathcal{M}$ and is
perfect w.r.t.\ its labelling.
\end{proposition}

\begin{proof}
Write $z_t := (s_t, u_t)$ for the pair of the hidden state and the
latent machine state, where
$u_{t+1} = \delta_u'\big(u_t, L'(s_t, a_t, s_{t+1})\big)$. The MDPRM
hypothesis makes $z_t$ a Markov chain and makes
$r_t = \delta_r'\big(u_t, L'(s_t, a_t, s_{t+1})\big)$ a fixed function
of $(z_t, a_t, s_{t+1})$ on every trajectory.

The machine tracks the Bayes posterior over this pair. For an
observable history $h$ of positive probability, let $b_h(z)$ be the
posterior probability of $z_t = z$ given $h$. This posterior does not
depend on the policy that produced $h$, because the probability of any
latent path together with $h$ factors into environment terms times the
policy terms $\prod_j \pi_j(a_j \mid h_j)$, the policy terms depend on
$h$ alone, and they cancel between the numerator and the denominator
of the conditional. By the same cancellation, the predicted
probability of the next observation,
$\Pr(o_{t+1} = o' \mid h_t, a_t)$, is a fixed function
$N(b_{h_t}, a_t, o')$ of the current posterior, obtained by
propagating $b_{h_t}$ through $p$ and $\omega$; and one step of Bayes
updating carries $b_h$ to a posterior $F(b_h, a, o')$ that again
depends only on $(b_h, a, o')$, obtained by propagating $b_h$ through
$p$ and the latent machine update, weighting by $\omega(o' \mid s')$,
and renormalizing by $N$. Consequently every reachable posterior
arises from a first-step posterior (the same formula with $\mu$ in
place of $b$) by iterating $F$, and by hypothesis the set
$\widehat{B}$ of reachable posteriors is finite.

The machine is the filter itself. Choose $P$ so that some injection
$\mathrm{code} : O \times A \times O \to 2^{P}$ exists and set
$L := \mathrm{code}$. The states are the initial $\iota$, the
posteriors in $\widehat{B}$, and a sink $\bot$ for updates of
probability zero, which loops with reward $0$; the transitions apply
$F$; and the reward on a transition is the posterior mean of the
latent reward,
\begin{align*}
\delta_r\big(b, &\mathrm{code}(o, a, o')\big)\\
  &:= \frac{1}{N(b, a, o')} \sum_{(s,u) \in Z} \sum_{s' \in S}
  b(s,u)\, p(s' \mid s, a)\\
  &\qquad\times \omega(o' \mid s')\,
  \delta_r'\big(u, L'(s, a, s')\big),
\end{align*}
with $\mu$ in place of $b$ out of $\iota$. Every table is computable
in exact arithmetic by forward closure from the first-step posteriors,
terminating because $\widehat{B}$ is finite.

The machine encodes. By induction on the update rule, on every history
of positive probability the machine state equals the posterior,
$x_t = b_{h_t}$, so its reward is
$\tilde r_t = \mathbb{E}\big[r_t \mid h_t, a_t, o_{t+1}\big]$. Taking
total expectations over the countably many triples
$(h_t, a_t, o_{t+1})$ gives
$\mathbb{E}^\pi[\tilde r_t] = \mathbb{E}^\pi[r_t]$ for every $t$ and
policy; all rewards are bounded by the largest entry of $\delta_r'$,
so both discounted sums converge absolutely and the returns agree.

The machine is perfect. Given $h_t$ and $a$, the joint distribution of
$(o_{t+1}, r_t)$ is obtained by propagating $b_{h_t}$ one step through
$p$, $\omega$, and the latent reward table, so it is a fixed function
of $(b_{h_t}, a) = (x_t, a)$. This is Eq.~\ref{eq:perfect_rm}.
\end{proof}
\end{defertext}

\begin{defertext}{}{The reference kernel of the greedy RM policy}
\begin{remark}[The reference kernel, and a caution about coarse
machines]
\label{rem:generalkernel}
Definitions~\ref{def:greedy} and~\ref{def:potential} apply to any
POMDPRM $\mathcal{T}$ encoding $\mathcal{M}$, perfect or not. Fix the
uniform reference policy $\pi_0$ and let
$d(o, u) := \sum_{t \ge 0} \gamma^t \Pr^{\pi_0}(o_t = o, x_t = u)$ be
the discounted occupancy of the pair $(o, u)$. For pairs with
$d(o, u) > 0$, let $\Pr_{\pi_0}(o' \mid o, u, a)$ be the weighted
average of the history conditionals $\Pr(o_{t+1} = o' \mid h_t, a)$
over all times $t$ and histories $h_t$ with $(o_t, x_t) = (o, u)$,
each weighted by $\gamma^t \Pr^{\pi_0}(h_t) / d(o, u)$; for
$d(o, u) = 0$ let the kernel be uniform. The history conditional does
not depend on the policy, by the cancellation argument of the
belief-filter proof applied to the plain POMDP, so the kernel is well
defined and depends only on the reference policy and $\gamma$.

When $\mathcal{T}$ is perfect, $\Pr(o_{t+1} = o' \mid h_t, a)$ is
constant across the histories in each $(o, u)$ cell (sum
Eq.~\ref{eq:perfect_rm} over reward values), so the kernel equals that
common value. It is then independent of the reference policy, the
discounting, and $t$, and coincides with the canonical
$\Pr(o_{t+1} \mid o_t, x_t, a_t)$. Since the full-support $\pi_0$
reaches every pair that any policy reaches, positive occupancy
coincides with reachability.

A caution accompanies the generality. A machine too coarse cannot
register a misranking its state does not represent. On the one-state
MysteryPath machine of Proposition~\ref{prop:nonperfect}, both actions
induce the same successor distribution from every pair, so every
$\phi$-greedy policy ties everywhere and $r$ comes out potentially
dense w.r.t.\ it, even though this reward undervalues exactly the
probing actions. Potential density w.r.t.\ a non-perfect encoder
therefore certifies nothing, and the classification only reads it
where the reward is structurally dense, where every encoder is perfect
and the kernel is canonical.
\end{remark}
\end{defertext}

\begin{defertext}{}{Classification of the MysteryPath abstraction}
The recipe from the main text asks for one machine that forgets and
still encodes. The one-state goal machine is that witness, and the two
histories ``went left, fell'' and ``went right, fell'' show it is not
perfect.

\begin{proposition}[Encoding does not imply perfection; the
MysteryPath abstraction is structurally sparse]
\label{prop:nonperfect}
There is a POMDP $\mathcal{M}$ (a one-cell abstraction of MysteryPath)
and a POMDPRM $\mathcal{T}$ encoding $\mathcal{M}$ whose RM has a
single non-terminal state, reproduces the reward step by step, and is
not perfect w.r.t.\ its labelling; moreover no POMDPRM with a single
non-terminal state that encodes $\mathcal{M}$ is perfect. By
Definition~\ref{def:structural}, $r$ is structurally sparse.
\end{proposition}
\begin{proof}
A hidden trap side $z \in \{L, R\}$ is drawn uniformly. The states are
$(\mathrm{start}, z)$ with observation $o_S$, $(\mathrm{fallen}, z)$
with observation $o_F$, $\mathrm{done}_1$ with observation $o_G$, and
$\mathrm{done}_2$ with observation $o_\bot$, and the actions are
$\{a_L, a_R\}$. From either start-type or fallen-type state with
hidden $z$, action $a_z$ leads to $(\mathrm{fallen}, z)$ with reward
$0$, and action $a_{\bar z}$ leads to $\mathrm{done}_1$ with reward
$+1$. $\mathrm{done}_1$ moves to $\mathrm{done}_2$ under any action,
$\mathrm{done}_2$ is absorbing, and all rewards from the done states
are $0$.

The machine $\mathcal{T}$ has one non-terminal state $u$, with
$P = \{\mathsf G\}$, $L(o,a,o') = \{\mathsf G\}$ iff $o' = o_G$,
$\delta_u \equiv u$, $\delta_r(u, \{\mathsf G\}) = 1$, and
$\delta_r(u, \varnothing) = 0$. The environment gives $+1$ exactly on
transitions into $\mathrm{done}_1$, whose arrival observation is
$o_G$, and $o_G$ occurs at no other time, so
$\tilde r_t = \mathbbm{1}[o_{t+1} = o_G] = r_t$ on every trajectory
and $\mathcal{T}$ encodes.

Now let $\mathcal{T}'$ be any encoder with a single non-terminal
state. Its RM state $x_t$ is constant, so perfection would force
$\Pr(o_{t+1}, r_t \mid h_t, a_t)$ to depend on $h_t$ only through
$o_t$. The histories $h = (o_S, a_L, o_F)$ and $h' = (o_S, a_R, o_F)$
both have positive probability under the uniform policy and share
$(o_t, x_t, a_t) = (o_F, u, a_L)$. But $h$ reveals $z = L$, so $a_L$
falls again with probability $1$, while $h'$ reveals $z = R$, so $a_L$
reaches the goal with probability $1$. The conditional distributions
differ, so no single-state machine is perfect.
\end{proof}
\end{defertext}

\begin{defertext}{}{Classification of the MemoryS13 cue abstraction}
The raw environment is \emph{not} structurally dense. The cue machine
of Figure~\ref{fig:rm-axis}(b) encodes the raw environment as well,
yet on the raw views it is not perfect, because inside a corridor
whose cells share a view the next observation depends on the
unobserved position, and by Definition~\ref{def:structural} one
non-perfect encoder makes the reward structurally sparse. That failure
belongs to the dynamics, since the reward never distinguishes corridor
positions. The abstraction below removes exactly this view-aliasing
and nothing else, and the density claim is about the abstraction. Here
the two-history recipe runs in reverse. Any machine that tries to
forget the cue and still match the returns of policies that differ
only in the final pick is caught, since matching pins its remaining
discounted reward at the shared state to two different values at once.

\begin{lemma}[Density certificate for the cue abstraction]
\label{lem:cueperfect}
Let $\mathcal{M}_{\mathrm{cue}}$ be the cue abstraction of
MiniGrid-MemoryS13. A hidden cue $c \in \{A,B\}$ is uniform; the
locations are start ($o_S$), cue room ($o_C^A$ or $o_C^B$, revealing
$c$), junction ($o_J$), and end ($o_E$, then absorbing $o_\bot$). At
the start, $\mathrm{peek}$ moves to the cue room and every other
action to the junction; every action leaves the cue room for the
junction; at the junction $\mathrm{pick}_i$ moves to the end with
reward $\mathbbm{1}[i = c]$ while other actions self-loop with reward
$0$; all other rewards are $0$. Then every POMDPRM that encodes
$\mathcal{M}_{\mathrm{cue}}$ is perfect w.r.t.\ its labelling, so the
abstraction's reward is structurally dense. The three-state cue
machine of Figure~\ref{fig:rm-axis}(b) is such an encoder, and it does
not reproduce the reward step by step.
\end{lemma}
\begin{proof}
Fix any encoder $\mathcal{T}$. For a positive-probability history $h$
ending at the junction, let $\kappa(h) \in \{A, B, \varnothing\}$
record whether $h$ contains $o_C^A$, contains $o_C^B$, or contains
neither, and let $x(h)$ be the machine state after reading the labels
of $h$.

We first show that $x(h)$ determines $\kappa(h)$. After a pick the
future is deterministic, so the machine's discounted reward from the
pick onward is a function of the state at the pick and of which pick
was played; write $G_i(x)$ for this quantity when $\mathrm{pick}_i$ is
played from state $x$, and $\Psi(x) := G_A(x) - G_B(x)$. Compare the
two deterministic policies that produce $h$ along its prefix, act
identically at every other history, and at $h$ play $\mathrm{pick}_A$
and $\mathrm{pick}_B$ respectively. Away from $h$ they generate
identical trajectories and label streams, so both return differences
are supported on the event that $h$ occurs. On that event the
environment returns differ by
$\gamma^{|h|}\Pr(h)\big(\Pr(c{=}A \mid h) - \Pr(c{=}B \mid h)\big)$
while the machine returns differ by $\gamma^{|h|}\Pr(h)\,\Psi(x(h))$,
the shared prefix cancelling. Encoding equates the two returns for
both policies, so
$$\Psi\big(x(h)\big) = \Pr(c{=}A \mid h) - \Pr(c{=}B \mid h),$$
which equals $1$, $-1$, or $0$ according as $\kappa(h)$ is $A$, $B$,
or $\varnothing$. Since $\Psi$ is a function of the state alone, two
junction histories with different classes cannot share a machine
state.

Perfection follows cell by cell. In a cell $\{o_t = o, x_t = u\}$ with
$o = o_S$ the only history is the empty one; with $o = o_C^X$ every
history has revealed $c = X$; with $o = o_J$ all histories share
$\kappa$ by the previous paragraph, and given $\kappa$ the pick reward
is Bernoulli with mean $\Pr(c = i \mid \kappa)$ while everything else
is deterministic; with $o \in \{o_E, o_\bot\}$ the process is
deterministic with reward $0$. In every case the conditional
distribution of $(o_{t+1}, r_t)$ is constant across the histories in
the cell, which is Eq.~\ref{eq:perfect_rm}.

Finally, the cue machine encodes $\mathcal{M}_{\mathrm{cue}}$, since
its reward on every transition is the average environment reward given
the history and the tower property gives equal expected returns for
every policy; it does not reproduce the reward step by step, since
from its no-cue state a pick is rewarded a fixed $\tfrac12$ while the
environment gives $0$ or $1$.
\end{proof}

\begin{proof}[Computation (potential sparsity of the cue abstraction)]
Take the cue machine, with rewards $1$ and $0$ for matching and
mismatching picks from $u_A, u_B$ and $\tfrac12$ for either pick from
$u_0$. At $(o_S, u_0)$ all one-step machine rewards are $0$, so the
greedy policy mixes uniformly over the four actions, peeking with
probability $\tfrac14$ and reaching the junction cue-blind with
probability $\tfrac34$; at $(o_J, u_0)$ it picks uniformly; at
$(o_J, u_X)$ it plays $\mathrm{pick}_X$; at the cue room all actions
give $0$ and lead to the junction. Hence
$$J(\pi) = \tfrac34 \cdot \gamma \cdot \tfrac12
+ \tfrac14 \cdot \gamma^2 = \tfrac38\gamma + \tfrac14\gamma^2.$$
With $\phi = V^\star$ as in the proof of Theorem~\ref{thm:greedy},
$J(\pi^{V^\star}) = J^\star = \gamma^2$ for $\gamma > \tfrac12$, since
peek, go, pick correctly then beats guessing immediately, worth
$\gamma \cdot \tfrac12$, and
$$J(\pi^{V^\star}) - J(\pi) = \tfrac38\gamma\,(2\gamma - 1) > 0,$$
so the abstraction's reward is potentially sparse w.r.t.\ the cue
machine.
\end{proof}
\end{defertext}

\begin{defertext}{}{Classification of the TinyReproduce abstraction}
TinyReproduce invites a shortcut. A mistake ends the episode, so one
might hope a machine could simply reward surviving play steps and
forget the tokens. The shortcut is not available, because episode end
is a consequence of correctness, not an input the machine can read.
The labelling function sees only $(o_t, a_t, o_{t+1})$, the play
observation is one constant view, and the final step shows that same
view whether the token was right or wrong. Two dictations that agree
on the tokens queried so far but disagree on the current target give,
under the same played token, identical $(o, a, o')$ and hence
identical labels with different rewards. The only way to reward a
token correctly is to already know it, and knowing every token at its
query means carrying the dictated sequence through the play phase, as
the sequence machine of Figure~\ref{fig:rm-axis}(c) does. The lemma
shows every encoder is forced into this shape.

\begin{lemma}[Density certificate for the dictation abstraction]
\label{lem:tinydense}
Consider the dictation abstraction of TinyReproduce matching the
implementation. A sequence $s \in \Sigma^k$ ($k \ge 2$) drawn
uniformly is shown token by token at times $0, \dots, k{-}1$ (the
observation at time $j$ is $s_j$, and watch-phase actions do not
affect the state), then queried in reverse. The observation at every
later step is one constant $o_{\mathrm{play}}$; an incorrect token
moves the process to an absorbing state, completion does too, and the
absorbing state also shows $o_{\mathrm{play}}$, so no observation
distinguishes a right token, a wrong token, or the end of the episode.
The dense variant gives $+1/k$ on each correct token; the sparse
variant gives $\tau/k$ on the transition into the absorbing state,
where $\tau$ is the number of correct tokens. Then, in the dense
variant, and in the sparse variant whenever $\gamma > (k-1)/k$, every
POMDPRM that encodes $\mathcal{M}$ resolves the dictated sequence, in
that its state at the end of the watch phase determines $s$, and is
perfect w.r.t.\ its labelling. Hence the abstraction's reward is
structurally dense in both variants.
\end{lemma}
\begin{proof}
Since the play and absorbing observations coincide, every label from
time $k$ on is $\ell_a := L(o_{\mathrm{play}}, a, o_{\mathrm{play}})$,
a function of the action alone; and since those observations carry no
information, a deterministic policy's actions from time $k$ on form a
fixed word $w \in \Sigma^{\omega}$ determined by the dictation it
watched. Conversely, every map $s \mapsto w(s)$, combined with any
watch behavior $\alpha$, is a deterministic policy, and every
dictation has positive probability.

Fix an encoder $\mathcal{T}$. For a machine state $\xi$ and a word
$w$, let
$G(\xi, w) := \sum_{j \ge 0} \gamma^j\, \delta_r(\xi_j, \ell_{w_j})$
with $\xi_0 = \xi$ and $\xi_{j+1} = \delta_u(\xi_j, \ell_{w_j})$, the
machine's discounted reward from time $k$ on, and let $E(s, w)$ be the
environment's. Writing $\rho(s,w)$ for the number of initial positions
at which $w$ matches the queried reversal of $s$, the dense variant
gives $E = \tfrac1k \sum_{j < \rho} \gamma^j$, and the sparse variant
gives $E = \gamma^{\rho}\,\rho/k$ on failure ($\rho < k$) and
$\gamma^{k-1}$ on completion. In both variants $E$ is strictly
increasing in $\rho$ for fixed $s$, in the sparse variant exactly when
$\gamma > (k-1)/k$, the same monotonicity as
Prop.~\ref{prop:redundancy}. Call the full-success value $F$; every
partial run is worth strictly less than $F$.

\emph{Step 1 (pinning).} Fix a watch behavior $\alpha$ and let $x(s)$
be the machine state at time $k$ under dictation $s$, a deterministic
function of $s$. Applying Definition~\ref{def:encode} to the policy
$(\alpha, w(\cdot))$ and to the same policy with the word changed at a
single dictation, then subtracting, gives
\begin{equation*}
G\big(x(s), w\big) - G\big(x(s), w'\big) = E(s, w) - E(s, w')
\tag{$\star$}
\end{equation*}
for all $s$, $w$, and $w'$. Suppose $x(s) = x(s'')$ with $s \ne s''$. By $(\star)$ the function
$w \mapsto E(s, w) - E(s'', w)$ is constant. Evaluating it at the
all-correct words $w_s$ and $w_{s''}$ and subtracting gives
$E(s'', w_s) + E(s, w_{s''}) = 2F$. But playing the correct word for
one dictation under the other fails at their first disagreement, so
both terms are strictly below $F$, a contradiction. Hence $x(\cdot)$
is one-to-one for every watch behavior.

\emph{Step 2 (perfection).} The same subtraction applied to
continuations pins every later state. For a machine state $\xi$
reached at a play time by some (dictation, action-prefix) pair,
$(\star)$ gives that any two pairs reaching $\xi$ have tail-reward
functions differing by a constant. A dead pair's tail is identically
$0$ while an alive pair's tail takes at least two values (all-correct
earns the full remainder, wrong-now earns $0$ or the current partial
amount), so dead and alive pairs never share a state. Now take two
alive pairs at positions $i$ and $i''$. In the sparse variant,
evaluating the constant difference at ``wrong now'' gives
$(i - i'')/k$ and at ``correct once, then wrong'' gives
$\gamma\,(i{+}1 - i''{-}1)/k$, and equating the two forces $i = i''$.
In the dense variant, ``wrong now'' earns $0$ for both pairs, so the
constant is $0$ and the tails are equal outright; evaluating at ``all
correct'' gives $\tfrac1k \sum_{j < k-i} \gamma^j = \tfrac1k
\sum_{j < k-i''} \gamma^j$, which again forces $i = i''$. With equal
positions the constant is $0$ in both variants, so the tail-reward
functions are identical, and this forces identical remaining targets,
since playing the correct word of one pair under the other would
otherwise make the two tails differ. Every reachable play-phase cell
$(o_{\mathrm{play}}, \xi, a)$ therefore contains pairs sharing
aliveness, position, and remaining targets, and the distribution of
$(o_{t+1}, r_t)$ on it is fixed, with next observation
$o_{\mathrm{play}}$ and the deterministic reward the variant assigns.
Watch-phase cells are constant as well, since the tokens of a uniform
$s$ are independent and uniform. Every cell satisfies
Eq.~\ref{eq:perfect_rm}, so the machine is perfect; since
$\mathcal{T}$ was arbitrary, the reward is structurally dense.
\end{proof}

\begin{remark}[Greedy behavior and the gate order]
\label{rem:tinygreedy}
W.r.t.\ the sequence machine (states are the remaining targets and the
position), the dense variant is potentially dense. The one-step reward
is $\tfrac1k$ for the correct token and $0$ otherwise at every
position, so the greedy RM policy is optimal and
Theorem~\ref{thm:greedy} predicts the null outcome directly. The
sparse variant is potentially sparse. At position $i \ge 1$ the
one-step reward is $i/k$ for a wrong token and $0$ for the correct
one, so the greedy policy fails on purpose, and $\phi = V^\star$
repairs it. Dense and potentially sparse would alone predict a helpful
bonus; the observed null is given by Prop.~\ref{prop:redundancy},
whose gate is read first.
\end{remark}
\end{defertext}

\begin{defertext}{}{Proof of Proposition \ref{prop:redundancy}}
\begin{proof}[Proof of Proposition~\ref{prop:redundancy}]
Fix the dictated sequence $s$ and a policy, and let $\tau$ be the
(random) number of correct tokens. The within-episode observation
stream is a deterministic function of $(s, \tau)$, since dictation is
scheduled, the reproduction observation is the constant
$o_{\mathrm{play}}$, and the episode ends at the first error or at
completion. Moreover, for the same $s$, the stream under $\tau$ is an
exact prefix of the stream under any $\tau' > \tau$, because the two
runs coincide through the step at which the token at position $\tau$
is played and the shorter run simply ends. Since each $b_t$ is a
nonnegative function of the prefix $o_{0:t}$, the cumulative
discounted bonus under $\tau'$ equals that under $\tau$ plus
nonnegative terms, so the intrinsic return is nondecreasing in $\tau$
for each fixed $s$.

The extrinsic return is strictly increasing in $\tau$ for each fixed
$s$. In the dense variant each additional correct token adds a
positive discounted term. In the sparse variant the terminal reward is
$\gamma^{\,t(\tau)}\tau/k$ with $t(\tau)$ increasing by one per
additional token, and
$\gamma^{\,t(\tau)+1}(\tau{+}1)/k > \gamma^{\,t(\tau)}\tau/k$ for all
$\tau < k$ exactly when $\gamma(\tau+1) > \tau$, which holds for all
$\tau \le k-1$ iff $\gamma > (k-1)/k$.

Now compare any two policies. A policy that fails to achieve
$\tau = k$ with positive probability is, on that event, strictly worse
extrinsically and weakly worse intrinsically than the policy that
corrects the corresponding tokens and agrees elsewhere. So for every
$\lambda \ge 0$ the optimal policies of $r + \lambda b$ and of $r$
coincide, and all optimal policies of $r$ have identical intrinsic
return.

\emph{Remark.} The hypothesis $b_t = g_t(o_{0:t})$ covers E3B exactly,
since its features and covariance are computed from the observations
received so far. NovelD scores the arrival observation $o_{t+1}$, so
the comparison acquires a single boundary term at the divergence step;
the conclusion then holds up to $\lambda$ times one first-visit bonus
term, consistent with the at-most-$\pm0.08$ effects observed. The
statement concerns the objective, not optimization dynamics, as does
Theorem~\ref{thm:greedy}.
\end{proof}
\end{defertext}

\begin{defertext}{}{Proof of Theorem \ref{thm:greedy}}
\begin{proof}[Proof of Theorem~\ref{thm:greedy}]
Structural density makes $\mathcal{T}$ perfect, so by
Remark~\ref{rem:generalkernel} the kernel of
Definition~\ref{def:greedy} equals the canonical
$\bar P(o' \mid o, u, a) := \Pr(o_{t+1}{=}o' \mid o_t{=}o, x_t{=}u,
a_t{=}a)$, which does not depend on the policy or on $t$ on reachable
pairs. Since $x_{t+1}$ is a deterministic function of
$(x_t, o_t, a_t, o_{t+1})$, the pair process is a finite MDP
$\mathcal{Y}$ on $Y = O \times U$ with reward
$\bar R(y, a, y') := \delta_r(u, \sigma)$, discount $\gamma$, and
initial distribution induced by $(\mu, \omega, u_0)$. Let $V^\star$,
$Q^\star$, $J^\star$ be its optimal value function, action-value
function, and optimal return, which exist and satisfy the Bellman
optimality equations because $\mathcal{Y}$ is finite and $\gamma < 1$.

First, $\sup_\pi J(\pi) = J^\star$. The machine state is a
deterministic function of the observation history, so every
observation-history policy induces a policy for $\mathcal{Y}$ with the
same machine return
$\tilde V(\pi) := \mathbb{E}^\pi[\sum_t \gamma^t \tilde r_t]$, and
conversely; and in a finite discounted MDP, history-dependent policies
do not outperform stationary ones, so
$\sup_\pi \tilde V(\pi) = J^\star$. By Definition~\ref{def:encode},
$\tilde V(\pi) = J(\pi)$ for every $\pi$, and the claim follows. No
step-by-step equality between $\tilde r_t$ and $r_t$ is used or
available, since encoding is an equality of expected returns, and
expected returns are all that optimal behavior depends on.

Second, take $\phi := V^\star$, a function on $U \times O$ as
Definition~\ref{def:potential} requires. At every pair $(y, a)$,
\begin{align*}
\sum_{o'} \bar P(o' \mid y, a)\big[\bar R(y, a, y')
  &+ \gamma V^\star(y')\big] - V^\star(y)\\
  &= Q^\star(y, a) - V^\star(y),
\end{align*}
so $\pi^{V^\star}$ is uniform over $\argmax_a Q^\star(y, \cdot)$, and
any randomization over the optimal-action set of a finite MDP is
optimal, hence $J(\pi^{V^\star}) = J^\star$.

Third, suppose the greedy RM policy $\pi$ had $J(\pi) < J^\star$. Then
$J(\pi) < J(\pi^{V^\star})$, and $\phi = V^\star$ would witness
potential sparsity w.r.t.\ $\mathcal{T}$, contradicting the
hypothesis. Hence $J(\pi) = J^\star = \sup_{\pi'} J(\pi')$, and the
greedy RM policy is optimal for $\mathcal{M}$.

\emph{What happened.} Perfection collapses the POMDP to a finite MDP
over the pairs (observation, machine state). In that MDP, shaping with
the optimal value function converts the one-step reward into the
optimal advantage, so myopic greed becomes optimal. A potentially
dense reward is one for which no such correction can help, which is
only possible when myopic greed was optimal already.
\end{proof}
\end{defertext}
\section{Discussion}
\label{sec:discussion}


We evaluated six memory architectures and two episodic bonuses under a single training stack across three environments. The identical bonus yielded three distinct outcomes: amplification on MysteryPath, equalization on MemoryS13, and a null effect on TinyReproduce.

\paragraph{Structural Determinants of Interaction Patterns.}
These divergence patterns isolate the post-exploration retention burden. The two sparsity properties separate them sequentially: structural sparsity applies first, while potential sparsity applies only when the reward is structurally dense. In MysteryPath, the bonus induces state coverage, but this yields zero return unless the high-dimensional path is retained. Structural sparsity (Definition~\ref{def:structural}) formalizes this lack of memory supervision: the abstraction admits a one-state machine that reproduces the return while omitting the path. Since retention is never reward-supervised, the bonus strictly multiplies the sequence-modeling capacity a cell already possesses (amplification). Conversely, the MemoryS13 cue abstraction is structurally dense but potentially sparse (Definition~\ref{def:potential}): the agent must actively discover the cue, but the subsequent retention burden is only a single reward-supervised bit. The bonus supplies the missing local advantage to seek the cue; since standard cells easily carry one bit, the bonus determines \emph{whether} convergence occurs rather than \emph{how well} (equalization). TinyReproduce's dictation abstraction is structurally dense and potentially dense: every encoder must carry the full sequence, and the observation stream is action-independent. Theorem~\ref{thm:greedy} accounts for this null effect; joint density implies the unshaped greedy policy is already optimal, leaving no exploratory deficit to correct.

\paragraph{Reward Manipulations and Exploratory Stagnation.}
Controlled variants confirm that reward structure, rather than density, drives bonus efficacy. The aligned reward (which directly supervises the required memory) renders the bonus redundant or detrimental, whereas the density-matched, uninformative distractor reward preserves its full restorative effect. Separately, the penalty variant isolates an optimization failure: a return-matched cost on exploratory actions drives architectures into stationary policies confined to zero-cost states. Either exploration bonus overcomes this suboptimal convergence regardless of its relative magnitude to the penalty (E3B pays $2.6\times$ the per-step cost and NovelD only $0.3\times$, yet both recover GatedDeltaNet equally). This indicates that recovery involves breaking stagnation rather than outbidding the penalty. Reward sparsity determines \emph{where} a bonus is effective, whereas reward pricing determines \emph{whether} exploration occurs at all.

\paragraph{Limitations.}
The formal results cover the regimes unevenly. Theorem~\ref{thm:greedy} accounts for the null pattern, whereas the amplify/equalize distinction rests on empirical retention-burden analysis. Additionally, every classification applies strictly to reward-preserving abstractions. The diagnostic requires reasoning over encoding machines, which may not exist globally or scale to high-dimensional inputs. Empirically, this study is bounded to discrete-action domains, a single on-policy algorithm, and observation-based episodic bonuses. Because the forward pass strictly separates the bonus and memory, we measure behavioral outcomes rather than latent encoding alterations. Evaluating these dependencies under global bonuses, off-policy learners, and continuous environments remains for future work.

\section{Conclusion}
\label{sec:conclusion}

The interaction between episodic exploration bonuses and neural memory architectures is task-dependent: an identical exploration signal amplifies capacity differences, equalizes performance, or yields a null effect according to the environment's reward structure. Controlled manipulations confirm that a dense reward neutralizes an exploration bonus only when it directly supervises the necessary latent representation, and that pricing exploratory actions can halt optimization even where the theoretical optimum is unchanged. Observation-anchored reward machines organize these regimes by what the reward supervises rather than how often it pays. Structural sparsity (an automaton reproduces the return while omitting the required memory) determines whether retention is supervised, dictating how severely architectural capacity gates the gain. Where the reward is structurally dense, potential sparsity (the one-step reward misprices local exploratory actions) determines whether an exploratory deficit exists. Practitioners must therefore evaluate what the reward inherently supervises and whether the architecture can retain the information the bonus exposes. An exploration bonus induces necessary state coverage, but only adequate representational capacity can convert that coverage into optimal expected return.
\newpage

\bibliography{aaai2027}

\appendix
\clearpage
\newpage
\ifpreprint
\section{Environment Details}
\label{app:envs}

\paragraph{Overview.}
We use three environments that vary \emph{how memory content is acquired} while keeping the training stack identical. Table~\ref{tab:envspec} summarizes the observation spaces, action spaces, horizons, and training budgets.

\begin{table*}[htbp]
\centering\small
\caption{Environment specifications. ``Flat dim'' is the flattened observation size provided to the encoder. MysteryPath uses a Nature-DQN convolutional encoder; the MiniGrid one-hot tensors and the TinyReproduce vectors are processed by an MLP encoder. Every configuration in the suite evaluates $n{=}10$ seeds.}
\label{tab:envspec}
\begin{tabular}{@{}lccccc@{}}
\toprule
Environment & Observation & Flat dim & Actions & $T_{\max}$ & Budget \\
\midrule
MysteryPath-Grid            & $84\times84\times3$ & 21{,}168 & 4 & 128 & 20M \\
MemoryS13 ($3\times3$)      & $3\times3\times20$  & 180      & 7 & 845 & 20M \\
MemoryS13 ($7\times7$)      & $7\times7\times20$  & 980      & 7 & 845 & 20M \\
TinyReproduce ($k{=}10$)    & $\mathbb{R}^{6}$    & 6        & 4 & 19  & 10M \\
\bottomrule
\end{tabular}
\end{table*}

\paragraph{MysteryPath-Grid \citep{pleines2023memorygym}.}
A continuous, invisible path connects a fixed origin to a goal. The agent observes its own position and the visible goal, but the path itself remains hidden. Stepping off the path results in a \emph{fall}, resetting the agent to the origin without altering the path's layout. Because the path is procedurally generated each episode, memory content must be discovered \emph{within} each episode through trial and error. An optimal agent retraces known safe steps and systematically probes unknown edges. Consequently, exploration incurs the cost of falls, and this exploration yields no benefit unless the verified path is retained in memory across resets. The action space consists of the four cardinal directions. Because falls are entirely avoidable, the penalty reward variant is strictly return-matched to the sparse baseline.

\paragraph{MiniGrid-MemoryS13 \citep{minigrid}.}
The agent spawns near a cue object (e.g., a key or a ball), must observe it, traverse a corridor, and select the matching object at a distant T-junction. Here, memory content exists from the start; the agent must merely find and retain it. Observations are one-hot encoded MiniGrid tensors. To appropriately handle categorical channels (object type, color, state) rather than continuous pixel intensities, a wrapper expands the raw $(N,N,3)$ grid into a $(N,N,20)$ binary format. We evaluate both the default $7\times7$ view and a restricted $3\times3$ view. The restricted view forces the agent to actively \emph{seek} the cue, introducing a mild exploration challenge, and thus serves as our primary discriminative setting.

\paragraph{The S13 reward wrapper.}
MiniGrid's default reward applies horizon discounting, which confounds task success with completion speed and is not strictly sparse. To ensure a fair comparison against MysteryPath using matched rewards and metrics, our wrapper provides three modes:
\begin{itemize}\itemsep2pt
  \item \textbf{native:} The default reward of $1 - 0.9\,t/T_{\max}$ for reaching the correct object, and $0$ otherwise. This is retained for baseline reference only.
  \item \textbf{matched sparse:} $+1$ for finding the matching object, and $0$ otherwise. This flat-sparse variant provides a direct match to MysteryPath's sparse reward.
  \item \textbf{matched penalty:} Identical to the matched sparse mode, but subtracts $p$ for every non-\texttt{nop} action. This introduces an \emph{unconditional} per-move cost absent from the native discount (which applies only upon success and cannot induce a policy freeze). MiniGrid's \texttt{done} action (index 6) is cost-free, acting as a zero-cost sanctuary. We set $p = 1/T_{\max} = 1/845$, serving as the horizon-normalized equivalent of MysteryPath's $-1/128$ off-path penalty.
\end{itemize}
Regardless of the reward mode, the wrapper computes \texttt{info["is\_success"]} based purely on environment geometry (\texttt{agent\_pos == success\_pos}). This guarantees that the reported success rate is a clean, binary metric that remains identical across all reward variants.

\paragraph{TinyReproduce (modified POPGym Autoencode, \citealp{morad2023popgym}).}
A sequence of $k$ tokens from an alphabet of $v$ symbols is dictated on a strict schedule, and the agent must reproduce it in reverse order. Observations are formatted as $[\,\mathbb{1}_{\text{watch}},\ \mathbb{1}_{\text{play}},\ \text{onehot}(v)\,]$, with a discrete action space over the $v$ symbols. For $k{=}10$ and $v{=}4$, the episode spans $2k-1 = 19$ steps. The environment begins by revealing the first token, followed by $k{-}1$ additional watch steps, and concludes with $k$ reproduction steps. The task fundamentally forces memory retention without any exploration demands, serving as a pure null-regime baseline.

\paragraph{TinyReproduce instrumentation and the \texttt{rm\_state} interface.}
The environment exposes its ground-truth reward-machine state at every step, enabling empirical verification of our structural classification:
\begin{itemize}\itemsep2pt
  \item \texttt{info["phase"]} $\in \{$\texttt{watch}, \texttt{play}$\}$.
  \item \texttt{info["rm\_state"]}: During the \emph{watch} phase, this contains the prefix revealed so far (\texttt{tuple(seq[:shown])}). During the \emph{play} phase, it contains the remaining sequence to be emitted, \emph{in emission order}. This is the exact minimal sufficient statistic of the history, making it the perfect RM state.
  \item \texttt{info["n\_correct"]} and \texttt{info["is\_correct"]} track step-wise accuracy, while \texttt{info["success"]} and \texttt{info["is\_success"]} indicate if the agent successfully reproduced all $k$ tokens.
\end{itemize}

\section{Hyperparameter Registry}
\label{app:hp}
\begin{table*}[!htpb]
\centering
\caption{Shared defaults (all environments) and per-architecture kwargs. Entries listed in the per-environment tables override these baseline settings.}
\label{tab:hp-shared}
\small
\begin{tabular}{@{}ll ll@{}}
\toprule
\multicolumn{2}{c}{Optimization} & \multicolumn{2}{c}{Rollout / PPO} \\
\cmidrule(lr){1-2}\cmidrule(lr){3-4}
learning rate      & $1\times10^{-4}$ & parallel envs      & 16 \\
epochs per update  & 4                & steps per env      & 512 \\
clip range         & 0.2              & discount $\gamma$  & 0.995 \\
entropy coef       & 0.008            & GAE $\lambda$      & 0.95 \\
value coef         & 1.0              & TBPTT chunk        & 64 \\
max grad norm      & 0.5              & chunks per batch   & 32 \\
target KL          & 0.05             & bonus coef $\lambda_{\text{int}}$ & 0.03 \\
encoder dim        & 128              & eval every         & 10 rollouts \\
encoder hidden     & 256              & eval episodes      & 20 \\
\midrule
\multicolumn{4}{c}{Architecture kwargs} \\
\cmidrule(lr){1-4}
GRU, LSTM, Memoryless & (defaults) & RetNet        & 4 heads \\
GatedDeltaNet & assoc.\ size 64 & Mamba-2 & 4 heads, $d_{\text{state}}$ 64, expand 2 \\
\bottomrule
\end{tabular}
\end{table*}

All runs are managed by a single declarative registry. A shared set of hyperparameters (Table~\ref{tab:hp-shared}) serves as the baseline across all experiments. The per-environment tables strictly detail deviations from this baseline. We do not use architecture-specific hyperparameters anywhere in the suite.

\begin{table*}[htbp]
\centering
\caption{MysteryPath-Grid. No hyperparameter overrides; the shared baseline from Table~\ref{tab:hp-shared} applies completely.}
\label{tab:hp-mpg}
\small
\begin{tabular}{@{}ll@{}}
\toprule
Environment id      & \texttt{MysteryPath-Grid-v0} \\
Horizon $T_{\max}$  & 128 \\
Discount $\gamma$   & 0.995 \\
Budget / seeds      & 20M steps, $n=10$ \\
Architectures       & GRU, LSTM, RetNet, GatedDeltaNet, Mamba-2, Memoryless \\
Training signals    & none, E3B, NovelD \\
HP overrides        & (shared) \\
\midrule
\multicolumn{2}{@{}l}{\emph{Reward variants}} \\
Sparse      & $+1$ on success \\
Penalty     & $+1$ on success, $-0.008$ ($-1/T_{\max}$) per off-path step \\
Aligned     & $+0.1$ per step of forward progress along the confirmed path \\
Distractor  & $+\varepsilon$ per revisited tile, $\varepsilon = 0.1/T_{\max} \approx 7.8\times10^{-4}$ \\
\bottomrule
\end{tabular}
\end{table*}

\begin{table*}[htbp]
\centering
\caption{MiniGrid-MemoryS13. The delayed cross-corridor reward necessitates a longer credit assignment horizon than the shared baseline provides. We tuned these four overrides specifically for this environment and applied them uniformly to all architectures.}
\label{tab:hp-s13}
\small
\begin{tabular}{@{}ll@{}}
\toprule
Environment id      & \texttt{MiniGrid-MemoryS13-v0} \\
Horizon $T_{\max}$  & 845 ($5 \cdot 13^2$) \\
Budget / seeds      & 20M steps, $n=10$ \\
Architectures       & GRU, LSTM, RetNet, GatedDeltaNet, Mamba-2, Memoryless \\
Training signals    & none, E3B, NovelD \\
\midrule
\multicolumn{2}{@{}l}{\emph{HP overrides} (all cells)} \\
discount $\gamma$   & 0.999 \\
GAE $\lambda$       & 0.98 \\
TBPTT chunk         & 32 \\
learning rate       & $3\times10^{-4}$ \\
\midrule
\multicolumn{2}{@{}l}{\emph{Reward variants}} \\
Sparse $3\times3$   & flat $+1$ on the matching object, $3\times3$ egocentric view \\
Sparse $7\times7$   & as above, default $7\times7$ view \\
Penalty $3\times3$  & flat $+1$, $-1/T_{\max} \approx -1.2\times10^{-3}$ per movement action \\
Distractor $3\times3$ & flat $+1$, $+\varepsilon$ per revisited cell, $\varepsilon = 0.1/T_{\max} \approx 1.2\times10^{-4}$ \\
\bottomrule
\end{tabular}
\end{table*}

\begin{table*}[htbp]
\centering
\caption{TinyReproduce. We use a single learning rate across \emph{all} architectures. Because the shared $10^{-4}$ rate caused several cells to floor near ${\approx}0.00$, the entire suite runs at $10^{-3}$ without any architecture-specific exceptions. The learning rate is the sole override; the discount factor relies on the shared $\gamma = 0.995$ baseline.}
\label{tab:hp-tiny}
\small
\begin{tabular}{@{}ll@{}}
\toprule
Environment id      & \texttt{TinyReproduce-v0} \\
Sequence            & $k = 10$ symbols, alphabet $v = 4$, reverse order \\
Horizon $T_{\max}$  & $2k-1 = 19$ \\
Budget / seeds      & 10M steps, $n=10$ \\
Architectures       & GRU, LSTM, RetNet, GatedDeltaNet, Mamba-2, Memoryless \\
Training signals    & none, E3B, NovelD \\
\midrule
\multicolumn{2}{@{}l}{\emph{HP overrides}} \\
All cells           & learning rate $10^{-3}$ (only override; $\gamma = 0.995$ shared) \\
\midrule
\multicolumn{2}{@{}l}{\emph{Reward variants} (exactly return-matched; differ only in timing)} \\
Dense  & $+1/k$ at each correctly reproduced token \\
Sparse & lump sum $(\#\text{correct})/k$ at episode end \\
\bottomrule
\end{tabular}
\end{table*}

\section{Reward-Variant Calibration}
\label{app:calibration}

\paragraph{Distractor: the discounted criterion.}
The distractor rewards the agent with $\varepsilon$ every time it revisits a tile. If a policy ignores the main task to farm these revisits, the theoretical maximum it can earn is $\varepsilon \sum_{t\ge0}\gamma^t = \varepsilon/(1-\gamma)$. For the optimal policy to remain unchanged, this farmed total must stay strictly below the $+1$ success payout. Evaluating the configurations: on MysteryPath, $\varepsilon/(1-\gamma) = (0.1/128)/0.005 \approx 0.156 < 1$. On MemoryS13, it is $(0.1/845)/0.001 \approx 0.118 < 1$. In both scenarios, the margin of safety is approximately $6$--$8\times$. This guarantees that farming revisits can never become the optimal strategy, even prior to factoring in the constraint that a tile must first be visited before it can pay out again.

\paragraph{Penalty: return-matching.}
Because the per-step penalty is $-1/T_{\max}$, the maximum cumulative penalty an agent can accrue over a full episode is $-1$. Since falls (MysteryPath) and moves beyond the free \texttt{done} action (MemoryS13) are entirely avoidable, the optimal trajectory for the sparse variant yields the exact same return under the penalty variant. The theoretical optimum is therefore mathematically unchanged; only the optimization landscape shifts, successfully isolating the exploration pricing dynamic.

\section{Bonus-Coefficient Sweep}
\label{app:lambda}

We use a coefficient of $\lambda_{\text{int}} = 0.03$ across all reported runs for both E3B and NovelD, regardless of the environment or reward variant. This value was established via the sweep in Table~\ref{tab:lambda}. To maintain consistency, NovelD and the non-MysteryPath environments inherit this value. Table~\ref{tab:lambda-lit} compares our chosen value against coefficients utilized in prior literature.

\begin{table}[htbp]
\centering\small
\caption{Bonus-coefficient sweep on MysteryPath (E3B) showing mean final success. The response scales monotonically with $\lambda$ for architectures that successfully convert the bonus into task success. It remains flat for GRU across an entire order of magnitude, indicating that GRU's lack of response is not a tuning failure.}
\label{tab:lambda}
\begin{tabular}{@{}lccc@{}}
\toprule
Cell & $\lambda{=}0.003$ & $\lambda{=}0.01$ & $\lambda{=}0.03$ \\
\midrule
GRU           & 0.175 & 0.120 & 0.133 \\
RetNet        & 0.300 & 0.370 & 0.475 \\
GatedDeltaNet & 0.225 & 0.280 & 0.292 \\
\bottomrule
\end{tabular}
\vspace{2pt}\par\footnotesize\raggedright
Note: $n$ represents hyperparameter combinations, not seeds. Each $\lambda$ is evaluated across the same six (learning rate, chunk length) pairings drawn from $\{10^{-4}, 2.5\times10^{-4}\} \times \{32, 64, 128\}$, ensuring $\lambda$ is not confounded by them ($\lambda{=}0.01$ covers 5 of the 6). Because this sweep utilized learning rates and chunk lengths differing from the final registry, it serves strictly to confirm tuning and a monotonic response, rather than acting as a matched-HP ablation. The final-registry results (Table 1 of the main text) outperform these sweep values due to an optimized configuration. We pull the success metric directly from \texttt{eval/mean\_reward}, which maps perfectly to the success rate on sparse MysteryPath.
\end{table}

\begin{table*}[htbp]
\centering\small
\caption{Published intrinsic-reward coefficients. Direct comparisons are difficult due to differing bonus scaling techniques. Notably, only E3B and our implementation normalize the bonus via a running standard deviation prior to applying the scale.}
\label{tab:lambda-lit}
\begin{tabular}{@{}llccl@{}}
\toprule
Method & Environment & Coef. & Norm.\ & Swept over \\
\midrule
E3B \citep{henaff2022e3b}    & MiniHack & 1.0 & yes & $10^{-4}$--$10$ \\
E3B                          & VizDoom  & $3\times10^{-6}$ & no & $3\times10^{-7}$--$10^{-4}$ \\
E3B                          & Habitat  & 0.1 & --- & $10^{-4}$--$1.0$ \\
NovelD \citep{zhang2021noveld} & MiniGrid & 0.1 & no & $\{0.01, 0.05, 0.1\}$ \\
NovelD                       & NetHack  & 100 & no & --- \\
\midrule
\textbf{Ours}                & MysteryPath, S13, Tiny & \textbf{0.03} & yes & $\{0.003, 0.01, 0.03\}$ \\
\bottomrule
\end{tabular}
\vspace{2pt}\par\footnotesize\raggedright
E3B does not report MiniGrid experiments; MiniHack provides the closest comparison and, similar to our setup, applies the coefficient to a running-std-normalized bonus. NovelD's coefficient applies to an unnormalized RND-difference bonus, complicating direct comparison. 
\end{table*}

\paragraph{Bracketing $\lambda$ from above.}
While Table~\ref{tab:lambda} confirms $\lambda = 0.03$ as the strongest tested value, we ran an upper bracket test at $\lambda = 0.1$ (a $3.3\times$ increase) on the two cells most responsive to the bonus (Table~\ref{tab:lambda-hi}). Both architectures completely collapsed, confirming that $\lambda = 0.03$ represents a true interior optimum rather than an unexplored edge.

\begin{table}[htbp]
\centering\small
\caption{Upper bracket test on $\lambda$. Evaluated using the sparse MysteryPath variant with E3B, measuring success at 5M steps. The $\lambda{=}0.03$ column reports mean $\pm$ std over the $n{=}10$ reported seeds, while $\lambda{=}0.1$ tracks a single seed per cell. Increasing the coefficient by $3.3\times$ completely destroys performance.}
\label{tab:lambda-hi}
\begin{tabular}{@{}lcc@{}}
\toprule
Cell & $\lambda{=}0.03$ ($n{=}10$) & $\lambda{=}0.1$ ($n{=}1$) \\
\midrule
RetNet        & $0.223 \pm 0.074$ & $0.00$ \\
GatedDeltaNet & $0.227 \pm 0.058$ & $0.02$ \\
\bottomrule
\end{tabular}
\vspace{2pt}\par\footnotesize\raggedright
These bracket runs utilize a single seed and terminate at 5M steps, and are not intended to replace a full sweep arm. However, the performance collapse is a full order of magnitude larger than the seed variance at $\lambda{=}0.03$, proving the drop is not an artifact of limited run time. Viewed alongside Table~\ref{tab:lambda}, the reported coefficient is bracketed on both sides.
\end{table}

\section{Learning Curves and Training Diagnostics}
\label{app:curves}

The learning dynamics across all environments and architectures illustrate the three patterns directly, bypassing reliance on final tail means alone. Curves represent the deterministic evaluation success rate (20 episodes per evaluation, taken every 10 rollouts), with the shaded band representing a 95\% seed-bootstrap confidence interval on the mean.

\begin{figure*}[htbp]\centering
\includegraphics[width=\textwidth]{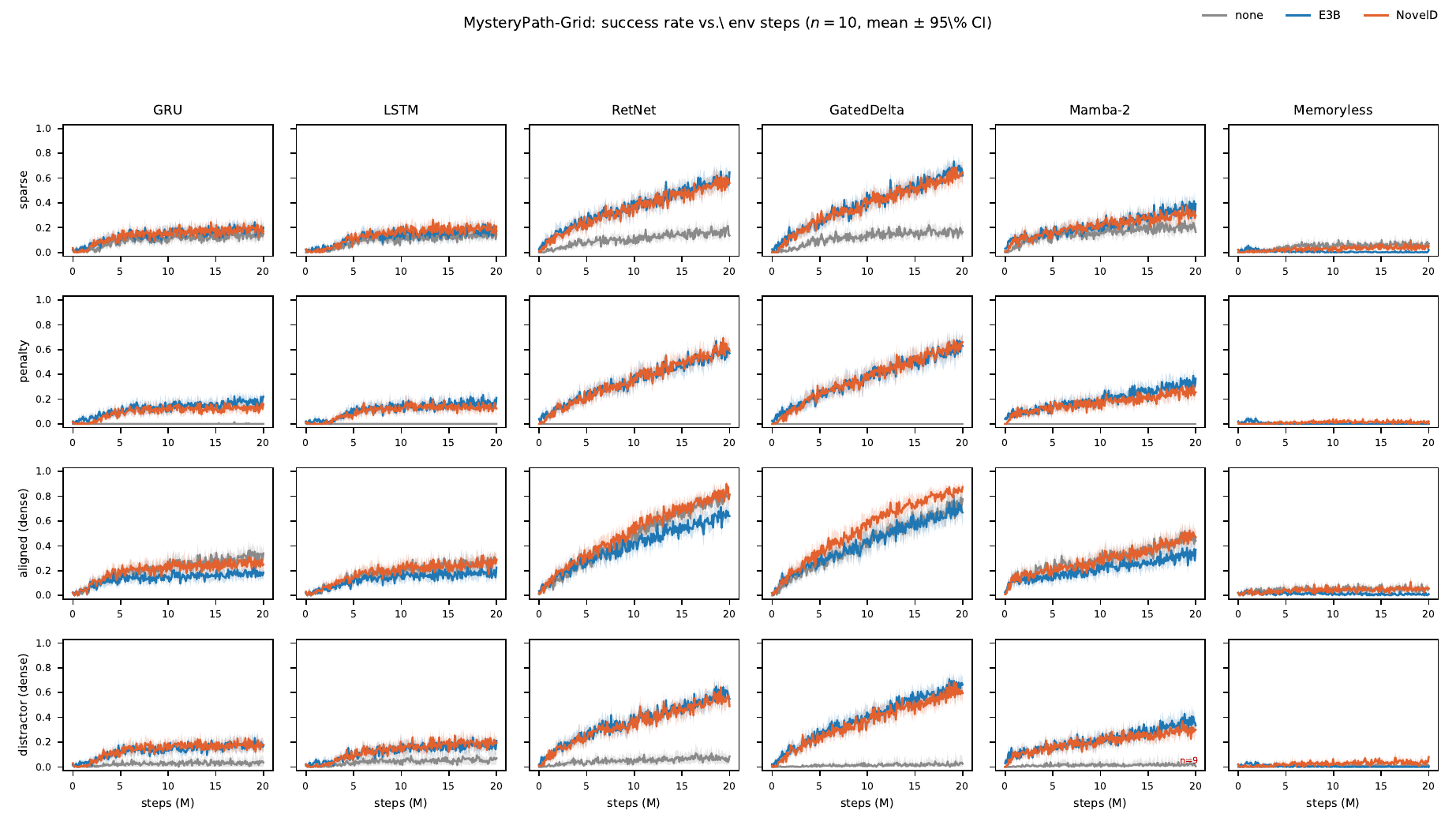}
\caption{MysteryPath-Grid across all four reward variants (rows) $\times$ six architectures (columns), $n{=}10$, showing the mean with a 95\% seed-bootstrap CI. Amplification is evident in the sparse row: the bonus separates RetNet, GatedDeltaNet, and Mamba-2 from a floored no-bonus baseline, while GRU and LSTM show minimal movement. The penalty row highlights the freeze: without a bonus, every architecture flatlines at zero, but bonuses successfully restore learning.}
\label{fig:curves-mpg}
\end{figure*}

\begin{figure*}[htbp]\centering
\includegraphics[width=\textwidth]{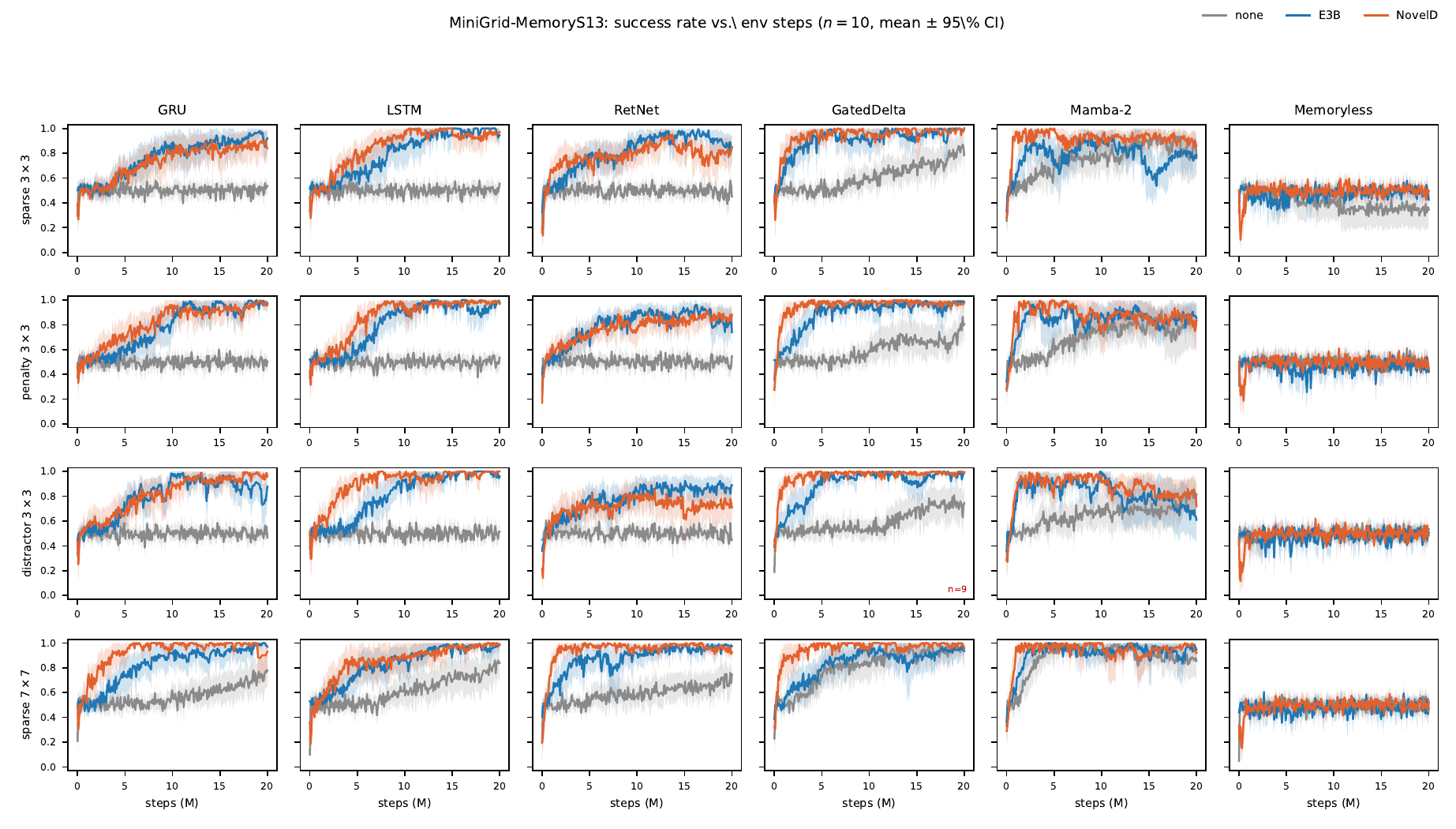}
\caption{MiniGrid-MemoryS13 across four reward variants (rows) $\times$ six architectures (columns), $n{=}10$, showing the mean with a 95\% seed-bootstrap CI. Equalization is visible in the sparse $3\times3$ row: GRU, LSTM, and RetNet are stuck at the cue-blind $0.50$ mark without a bonus, but reach the ceiling with one. GatedDeltaNet and Mamba-2 gain proportionally less. The wide confidence bands without a bonus reflect the bimodal ``seed lottery'' discussed in Section~\ref{app:faq}.}
\label{fig:curves-s13}
\end{figure*}

\begin{figure*}[htbp]\centering
\includegraphics[width=\textwidth]{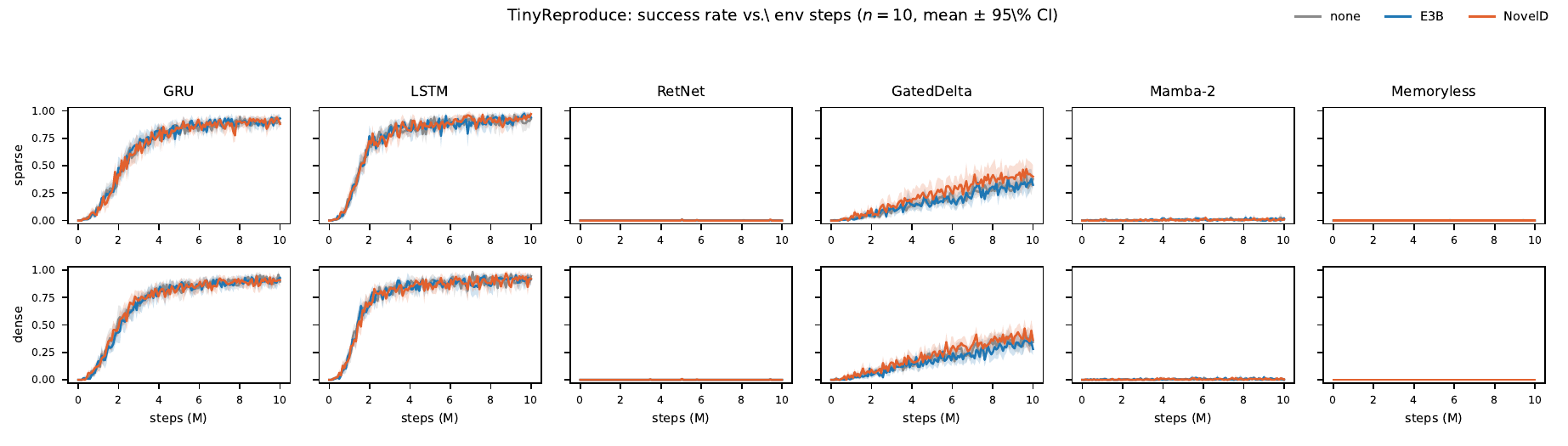}
\caption{TinyReproduce across both reward variants (rows) $\times$ six architectures (columns), $n{=}10$, showing the mean with a 95\% seed-bootstrap CI. The null pattern is consistent: all three signals overlap perfectly in every panel. The only observed differences are architectural (GRU and LSTM solve the task; RetNet and Mamba-2 do not), and adding a bonus does not alter the outcome.}
\label{fig:curves-tiny}
\end{figure*}

\begin{figure}[htbp]\centering
\includegraphics[width=\linewidth]{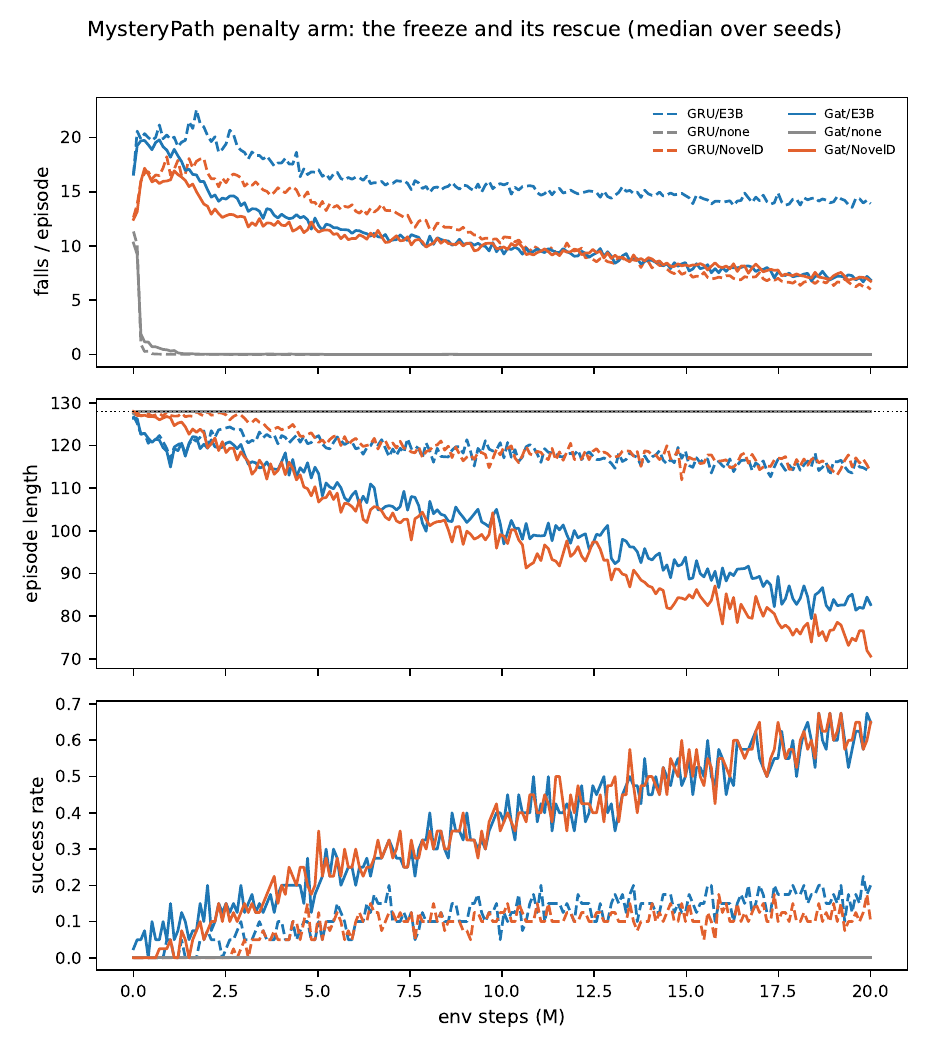}
\caption{The penalty-variant freeze on MysteryPath (median across seeds). Without a bonus, falls per episode drop to zero, episode length is pinned at $T_{\max}{=}128$, and the success rate remains at zero. The agent abandons the task to avoid penalties. Policy entropy does \emph{not} collapse, indicating a sampling failure rather than premature convergence. The agent mixes within the safe region but never observes the sparse $+1$ reward. Both bonuses successfully break this freeze.}
\label{fig:penalty-diag}
\end{figure}

\begin{table}[htbp]\centering\small
\caption{MysteryPath penalty arm: falls per episode and terminal success (mean $\pm$ std over seeds). Without a bonus, the agent drives falls to zero, runs out the clock, and never reaches the goal. The bonus arms trade falls for task success, proving the freeze is a policy that stops paying the exploration cost, not one that forgot how to learn.}
\label{tab:falls}
\begin{tabular}{@{}llcc@{}}\toprule
Arch & Bonus & Falls / episode & Success \\\midrule
GRU & none & $0.0\pm0.0$ & $0.00\pm0.00$ \\
GRU & E3B & $14.0\pm0.7$ & $0.22\pm0.08$ \\
GRU & NovelD & $6.1\pm0.6$ & $0.16\pm0.09$ \\
\midrule
GatedDeltaNet & none & $0.0\pm0.0$ & $0.00\pm0.00$ \\
GatedDeltaNet & E3B & $7.0\pm0.8$ & $0.63\pm0.17$ \\
GatedDeltaNet & NovelD & $6.7\pm0.6$ & $0.66\pm0.10$ \\
\bottomrule\end{tabular}
\end{table}

\paragraph{The freeze is a sampling failure, not a value-estimation failure.}
The timing of the first success rules out the possibility that the agent reaches the goal but fails to propagate its value backward. Across every no-bonus seed in the penalty arm, the goal is \emph{never} reached during the 20M steps. The $+1$ reward is never observed, leaving nothing to propagate. Upon introducing a bonus, the agent hits the goal within the first few thousand steps.

\begin{table}[htbp]\centering\small
\caption{The first environment step where any evaluation episode reached the goal (MysteryPath penalty arm, per seed). ``Never'' indicates the goal was not reached during the 20M-step run. The no-bonus policy simply never acquires the data required to begin learning.}
\label{tab:firstgoal}
\begin{tabular}{@{}llc@{}}\toprule
Arch & Bonus & First success (env steps, per seed) \\
\midrule
GRU & none   & never, never, never \\
GRU & E3B    & 2{,}774 \quad 409 \quad 409 \\
GRU & NovelD & 3{,}481 \quad 337 \quad 9{,}376 \\
\midrule
GatedDeltaNet & none   & never, never, never \\
GatedDeltaNet & E3B    & 1{,}984 \quad 445 \quad 445 \\
GatedDeltaNet & NovelD & 2{,}105 \quad 4{,}703 \quad 5{,}569 \\
\bottomrule\end{tabular}
\end{table}

\paragraph{The rescue is not arithmetic.}
The bonus does not simply overpower the penalty. While E3B delivers $2.6\times$ the penalty scale, NovelD only delivers $0.3\times$, meaning the penalty heavily outweighs the NovelD bonus (Table~\ref{tab:magnitude}). Nevertheless, NovelD successfully rescues the freeze (Table~\ref{tab:falls}), proving the rescue relies on more than the net positive sign of the per-step reward.

\begin{table}[htbp]\centering\small
\caption{The delivered per-step intrinsic reward $\lambda \cdot b$ compared to the per-step penalty $1/T_{\max}$, averaged over training. The logged bonus is the raw, pre-coefficient value; the delivered quantity is $\lambda$ times that ($\lambda = 0.03$). The final column shows the ratio of this bonus to the penalty magnitude.}
\label{tab:magnitude}
\begin{tabular}{@{}llccc@{}}\toprule
Environment / variant & Bonus & raw $b$ & $\lambda b$ & vs.\ penalty \\
\midrule
MPG / sparse   & E3B    & 0.689 & 0.0207 & $2.6\times$ \\
MPG / sparse   & NovelD & 0.067 & 0.0020 & $0.3\times$ \\
MPG / penalty  & E3B    & 0.678 & 0.0203 & $2.6\times$ \\
MPG / penalty  & NovelD & 0.075 & 0.0023 & $0.3\times$ \\
\midrule
S13 / sparse $3\times3$ & E3B    & 0.648 & 0.0194 & $2.5\times$ \\
S13 / sparse $3\times3$ & NovelD & 0.016 & 0.0005 & $0.1\times$ \\
S13 / penalty $3\times3$ & E3B    & 0.647 & 0.0194 & $2.5\times$ \\
S13 / penalty $3\times3$ & NovelD & 0.017 & 0.0005 & $0.1\times$ \\
\midrule
Tiny / sparse  & E3B    & 0.498 & 0.0149 & --- \\
Tiny / sparse  & NovelD & 0.000 & 0.0000 & --- \\
Tiny / dense   & E3B    & 0.472 & 0.0142 & --- \\
Tiny / dense   & NovelD & 0.000 & 0.0000 & --- \\
\bottomrule\end{tabular}
\vspace{2pt}\par\footnotesize\raggedright
In TinyReproduce, NovelD's delivered bonus is numerically zero. Its episodic first-visit gate never triggers because the play-phase observation is a single constant view. Here, the null result for NovelD is structurally guaranteed, serving as an independent confirmation of redundancy.
\end{table}

\begin{figure*}[htbp]\centering
\includegraphics[width=\textwidth]{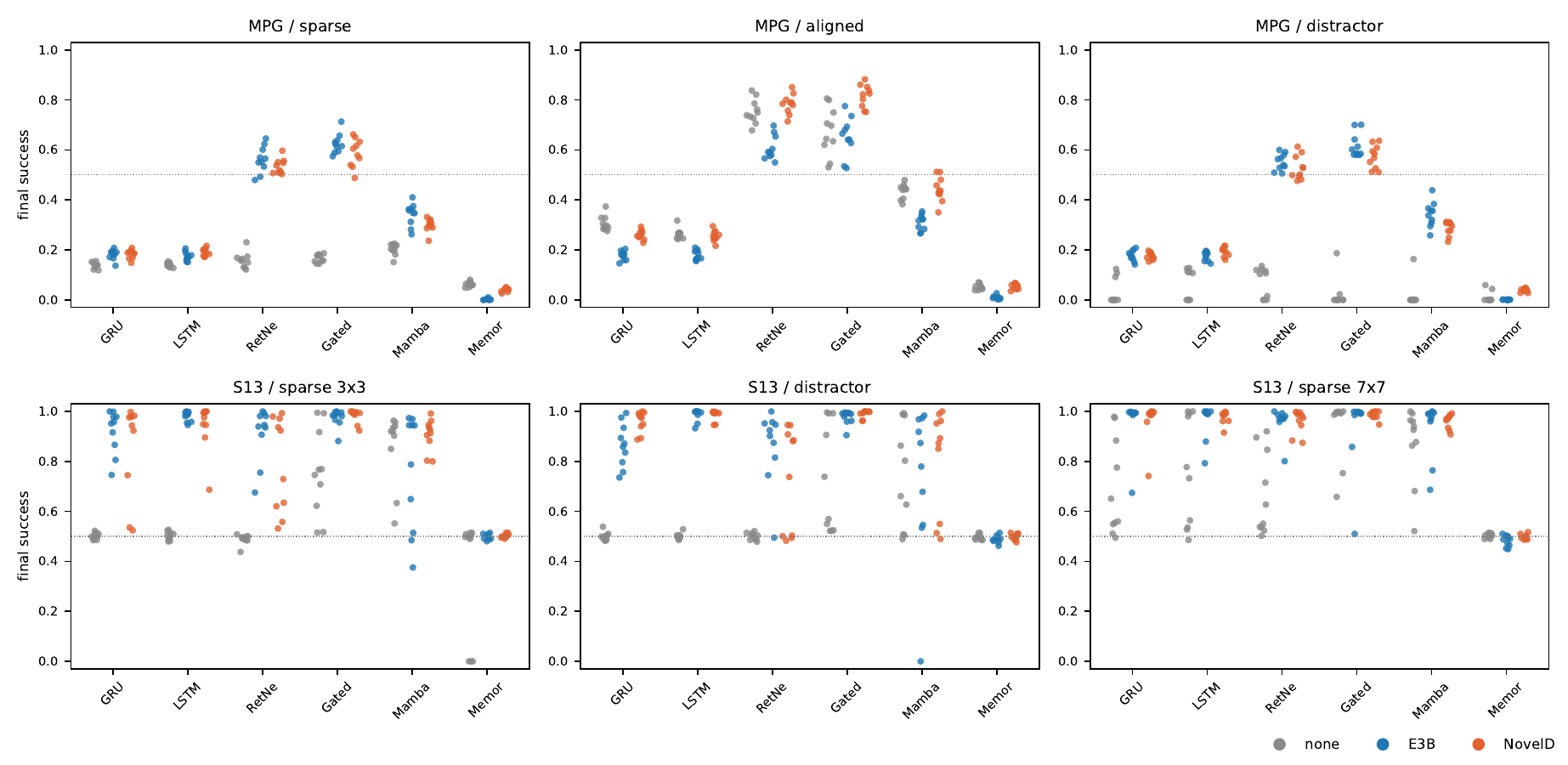}
\caption{Per-seed tail-mean success for every cell in the suite. The bimodal nature of MemoryS13 necessitates the scale-free summary in Section~\ref{app:stats}. Seeds either land at the cue-blind $0.50$ mark or hit the ceiling; a mean midway between the two misrepresents actual runs.}
\label{fig:seed-strip}
\end{figure*}

\section{Statistical Analysis of the Architecture--Bonus Interaction}
\label{app:stats}

This section quantifies the Architecture$\times$Bonus interaction claims (amplification, equalization, and null). The primary metric is the per-seed tail-mean success rate over the final 20\% of training. We exclude the memoryless control to avoid artificially inflating the between-architecture spread statistic.

\paragraph{Factorial analysis.}
Table~\ref{tab:anova} details a two-way ANOVA for each environment. Due to highly heterogeneous cell variances and bounded, occasionally bimodal outcomes (Fig.~\ref{fig:seed-strip}), we pair each interaction $F$ with a permutation $p$-value. Permuting bonus labels within each architecture ($2\times10^4$ resamples) preserves the architecture main effect while nullifying the bonus effect. The permutation values offer a robust, non-parametric confirmation.

\begin{table*}[htbp]
\centering\small
\caption{Two-way ANOVA (Architecture $\times$ Bonus) on per-seed tail means across the five memory architectures ($n{=}10$ per cell, $\mathrm{df} = 4, 1, 4, 90$). Both gridworlds demonstrate a massive, statistically significant interaction under both bonuses. TinyReproduce shows no interaction, perfectly confirming the null effect.}
\label{tab:anova}
\begin{tabular}{@{}llccc@{}}
\toprule
& & \multicolumn{3}{c}{$F$ \quad ($p_{\mathrm{param}}$) \quad [$p_{\mathrm{perm}}$]} \\
\cmidrule(lr){3-5}
Environment & Effect & E3B & & NovelD \\
\midrule
MysteryPath & Architecture         & 271.4 \;\;($2.6{\times}10^{-49}$) & & 274.6 \;\;($1.6{\times}10^{-49}$) \\
(sparse)    & Bonus                & 1300.1 \;($2.8{\times}10^{-55}$) & & 1356.7 \;($4.6{\times}10^{-56}$) \\
            & Arch.\ $\times$ Bonus & \textbf{226.4} \;[$0.0002$]      & & \textbf{234.7} \;[$0.0002$] \\
\midrule
MemoryS13   & Architecture         & 7.6 \;\;($2.4{\times}10^{-5}$)   & & 17.9 \;\;($7.7{\times}10^{-11}$) \\
($3\times3$)& Bonus                & 162.7 \;($7.0{\times}10^{-22}$)  & & 137.3 \;($8.4{\times}10^{-20}$) \\
            & Arch.\ $\times$ Bonus & \textbf{22.1} \;[$0.00005$]      & & \textbf{8.0} \;[$0.021$] \\
\midrule
Tiny        & Architecture         & 2308.1 \;($9.4{\times}10^{-90}$) & & 1414.2 \;($2.6{\times}10^{-80}$) \\
(sparse)    & Bonus                & 0.01 \;\;($0.93$)                & & 2.6 \;\;($0.11$) \\
            & Arch.\ $\times$ Bonus & 0.35 \;[$0.61$]                  & & 1.7 \;[$0.23$] \\
\midrule
Tiny        & Architecture         & 1984.0 \;($7.9{\times}10^{-87}$) & & 2751.1 \;($3.7{\times}10^{-93}$) \\
(dense)     & Bonus                & 0.42 \;\;($0.52$)                & & 0.6 \;\;($0.43$) \\
            & Arch.\ $\times$ Bonus & 0.73 \;[$0.58$]                  & & 1.2 \;[$0.29$] \\
\bottomrule
\end{tabular}
\vspace{2pt}\par\footnotesize\raggedright
Permutation $p$-values are reported exclusively for the interaction term. The TinyReproduce architecture main effect is enormous, while both bonus terms are indistinguishable from zero: the cells perform differently, but the bonus does not affect them.
\end{table*}

\paragraph{Effect sizes.}
With $n{=}10$, the gridworld interactions are overwhelmingly significant; thus, we focus on effect sizes using seed-level bootstrap intervals ($10^4$ resamples). These intervals properly capture the confirmatory null by measuring the absence of an effect.

\paragraph{Amplification and equalization as dispersion shifts.}
Let $V_{\mathrm{cond}}$ represent the variance across the five architecture means. Amplification predicts $V_{\mathrm{bonus}}/V_{\mathrm{none}} > 1$, equalization predicts ${<}1$, and the null predicts ${=}1$.

\begin{table*}[htbp]
\centering\small
\caption{Between-architecture variance of cell means, no bonus vs.\ bonus, with seed-level bootstrap 95\% CIs on the ratio. The bonus increases dispersion on MysteryPath roughly fifty-fold, shrinks it on MemoryS13 by four-fold, and leaves it untouched on TinyReproduce.}
\label{tab:dispersion}
\begin{tabular}{@{}lcccccc@{}}
\toprule
& & \multicolumn{2}{c}{E3B} & \multicolumn{2}{c}{NovelD} \\
\cmidrule(lr){3-4}\cmidrule(lr){5-6}
Environment & $V_{\mathrm{none}}$ & $V_{\mathrm{e3b}}$ & Ratio [95\% CI]
            & $V_{\mathrm{nvld}}$ & Ratio [95\% CI] \\
\midrule
MysteryPath (sparse)    & 0.0007 & 0.0440 & 63.2 [42.1, 101.6]
                                 & 0.0363 & 52.1 [34.6, 83.8] \\
MemoryS13 ($3\times3$)  & 0.0300 & 0.0080 & 0.27 [0.06, 0.84]
                                 & 0.0057 & 0.19 [0.06, 0.59] \\
TinyReproduce (sparse)  & 0.2046 & 0.2099 & 1.03 [0.99, 1.06]
                                 & 0.2059 & 1.01 [0.98, 1.04] \\
TinyReproduce (dense)   & 0.2070 & 0.2086 & 1.01 [0.98, 1.04]
                                 & 0.2037 & 0.98 [0.95, 1.02] \\
\bottomrule
\end{tabular}
\vspace{2pt}\par\footnotesize\raggedright
Variance is calculated across the five memory architectures' cell means ($\mathrm{ddof}{=}1$). The TinyReproduce intervals rule out a dispersion shift of more than $\pm6\%$, confirming a precisely measured absence of an effect.
\end{table*}

\paragraph{Per-cell effect sizes.}
Table~\ref{tab:poi} outlines the probability of improvement $\Pr(Y_{\mathrm{bonus}} > Y_{\mathrm{none}})$ \citep{agarwal2021deep} across all seed pairings.

\begin{table*}[htbp]
\centering\small
\caption{Probability of improvement under each bonus vs.\ no bonus, per cell (E3B / NovelD). Values near $1$ mean the bonus helps on essentially every seed pairing; values near $0.5$ indicate no effect.}
\label{tab:poi}
\begin{tabular}{@{}lcccc@{}}
\toprule
Cell & MysteryPath & MemoryS13 ($3\times3$) & Tiny (sparse) & Tiny (dense) \\
\midrule
GRU           & 0.95 / 0.97 & 1.00 / 1.00 & 0.47 / 0.56 & 0.37 / 0.35 \\
LSTM          & 0.98 / 1.00 & 1.00 / 1.00 & 0.73 / 0.68 & 0.70 / 0.56 \\
RetNet        & 1.00 / 1.00 & 1.00 / 1.00 & 0.40 / 0.45 & 0.55 / 0.55 \\
GatedDeltaNet & 1.00 / 1.00 & 0.85 / 0.90 & 0.38 / 0.61 & 0.41 / 0.54 \\
Mamba-2       & 1.00 / 1.00 & 0.48 / 0.55 & 0.45 / 0.44 & 0.64 / 0.54 \\
\bottomrule
\end{tabular}
\end{table*}

On MemoryS13, Mamba-2 ($0.48$ under E3B) and GatedDeltaNet ($0.85$) show minimal improvement because they already clear the cue-blind optimum without a bonus. Equalization closes the spread from below: the bonus lifts the stuck cells and does practically nothing for those already succeeding. On TinyReproduce, the per-cell values scatter around $0.5$, indicating a genuine null result.

\paragraph{A scale-free reading of MemoryS13.}
Because the seed distribution on MemoryS13 is bimodal, a binary summary provides a more faithful picture. We count a seed as successful if its tail mean exceeds $0.6$.

\begin{table}[htbp]
\centering\small
\caption{MemoryS13 ($3\times3$ sparse): number of seeds out of ten successfully learning the task (tail mean $> 0.6$). This summary is immune to monotonic rescaling of the success rate.}
\label{tab:s13binary}
\begin{tabular}{@{}lccc@{}}
\toprule
Cell & none & E3B & NovelD \\
\midrule
GRU           & 0/10 & 10/10 & 8/10 \\
LSTM          & 0/10 & 10/10 & 10/10 \\
RetNet        & 0/10 & 10/10 & 8/10 \\
GatedDeltaNet & 8/10 & 10/10 & 10/10 \\
Mamba-2       & 9/10 & 7/10  & 10/10 \\
\bottomrule
\end{tabular}
\end{table}

The three failing cells learn on every seed (E3B) or most seeds (NovelD) once a bonus is introduced. The two successful cells mostly maintain their performance. This captures the equalization pattern without relying on the numerical scale.

\paragraph{A caution on the raw scale.}
MemoryS13 compression is sensitive to the choice of scale (e.g., a logit transform yields an E3B dispersion ratio near $1.0$). Including the memoryless control eliminates compression because the control stays flat while memory cells hit the ceiling. Thus, the equalization claim applies to the raw success-rate scale among memory architectures, supported by Table~\ref{tab:s13binary}. Amplification on MysteryPath holds regardless of scale.

Table~\ref{tab:sep} makes this dependence explicit on the standard-deviation scale. MysteryPath widens under both bonuses across sparse and distractor variants. MemoryS13's compression is significant among the five memory cells ($-0.075$ and $-0.090$) but vanishes when the memoryless control is added ($-0.003$).

\begin{table*}[htbp]\centering\small
\caption{Architecture separation index ($\mathrm{SEP}$). $\Delta$SEP compares the bonus arm against \emph{none}, with a 95\% CI. A positive value means the bonus widens the performance gap; a negative value means it compresses it.}
\label{tab:sep}
\begin{tabular}{@{}lcccc@{}}\toprule
Env / reward & SEP(none) & SEP(e3b) & SEP(nvld) & $\Delta$SEP(e3b) 95\% CI \\
\midrule
\multicolumn{5}{@{}l}{\emph{5 memory cells}} \\
MPG / sparse & 0.024 & 0.188 & 0.170 & $+0.164$ [$+0.153$, $+0.175$] \\
MPG / aligned & 0.196 & 0.206 & 0.245 & $+0.010$ [$-0.009$, $+0.030$] \\
MPG / distractor & 0.016 & 0.184 & 0.167 & $+0.169$ [$+0.146$, $+0.178$] \\
S13 / sparse $3\times3$ & 0.155 & 0.080 & 0.068 & $-0.075$ [$-0.128$, $-0.013$] \\
S13 / distractor & 0.116 & 0.095 & 0.106 & $-0.021$ [$-0.085$, $+0.059$] \\
S13 / sparse $7\times7$ & 0.104 & 0.014 & 0.010 & $-0.090$ [$-0.127$, $-0.035$] \\
\midrule
\multicolumn{5}{@{}l}{\emph{6 cells incl.\ Memoryless}} \\
MPG / sparse & 0.043 & 0.221 & 0.196 & $+0.178$ [$+0.170$, $+0.188$] \\
MPG / aligned & 0.241 & 0.235 & 0.281 & $-0.007$ [$-0.024$, $+0.011$] \\
MPG / distractor & 0.020 & 0.218 & 0.192 & $+0.198$ [$+0.180$, $+0.207$] \\
S13 / sparse $3\times3$ & 0.173 & 0.170 & 0.159 & $-0.003$ [$-0.049$, $+0.041$] \\
S13 / distractor & 0.112 & 0.170 & 0.177 & $+0.058$ [$+0.016$, $+0.108$] \\
S13 / sparse $7\times7$ & 0.142 & 0.175 & 0.176 & $+0.033$ [$+0.003$, $+0.055$] \\
\bottomrule\end{tabular}
\vspace{2pt}\par\footnotesize\raggedright
On the MPG \emph{aligned} row, the bonus stops widening the spread ($+0.010$, interval crosses zero). Amplification turns off exactly where the reward is no longer structurally sparse.
\end{table*}

\paragraph{The aligned variant acts as an internal control for amplification.}
The core amplification claim states a bonus widens the gap between memory architectures \emph{only when the reward is structurally sparse}. The aligned variant tests this while holding the environment constant. Changing only the reward from sparse to dense-aligned causes amplification to vanish entirely ($\Delta$SEP plummets from $+0.164$ to $+0.010$). Conversely, the distractor variant (dense but structurally sparse) behaves like the sparse variant ($+0.169$). Reward structure, not the environment itself, dictates the pattern.

\paragraph{Reward structure as a core factor, holding the environment constant.}
We run a comprehensive three-way factorial—Architecture $\times$ Bonus $\times$ Reward structure—within MysteryPath to isolate reward structure from environmental confounds.

\begin{table}[htbp]
\centering\small
\caption{Three-way ANOVA within MysteryPath: 5 architectures $\times$ 2 bonuses $\times$ 4 reward variants $\times$ $10$ seeds. The significant three-way term confirms the architecture$\times$bonus interaction is heavily influenced by reward structure.}
\label{tab:threeway}
\begin{tabular}{@{}lrrrr@{}}
\toprule
Effect & SS & df & $F$ & $\omega^2$ \\
\midrule
Architecture (A)                  & 5.732 & 4  & 842.9  & 0.274 \\
Bonus (B)                         & 3.975 & 1  & 2338.3 & 0.190 \\
Reward structure (C)              & 3.962 & 3  & 776.8  & 0.189 \\
A $\times$ B                      & 1.920 & 4  & 282.4  & 0.091 \\
A $\times$ C                      & 0.893 & 12 & 43.8   & 0.042 \\
B $\times$ C                      & 3.274 & 3  & 641.9  & 0.156 \\
\textbf{A $\times$ B $\times$ C}  & 0.553 & 12 & \textbf{27.1} & \textbf{0.025} \\
\midrule
residual                          & 0.612 & 360 & & \\
\bottomrule
\end{tabular}
\vspace{2pt}\par\footnotesize\raggedright
Reward structure drives as much variance as the bonus ($0.189$ vs.\ $0.190$), and the bonus$\times$reward term ($0.156$) heavily outweighs the architecture$\times$bonus term ($0.091$). The reward given matters as much as the bonus itself.
\end{table}

\begin{table*}[htbp]
\centering\small
\caption{Breaking down the architecture$\times$bonus interaction within MysteryPath. The interaction is massive in the three structurally sparse variants, but plummets in the dense variant.}
\label{tab:byvariant}
\begin{tabular}{@{}lccrl@{}}
\toprule
Reward variant & Classification & $F(\mathrm{A}\!\times\!\mathrm{B})$
 & partial $\omega^2$ & Dispersion ratio \\
\midrule
sparse     & struct.\ sparse & 226.4 & 0.900 & $63.2$ \\
penalty    & struct.\ sparse & 249.9 & 0.909 & undefined$^\dagger$ \\
distractor & struct.\ sparse & 76.4  & 0.751 & $136.5$ \\
aligned    & struct.\ dense  & 5.2   & 0.144 & $1.1$ \\
\bottomrule
\end{tabular}
\vspace{2pt}\par\footnotesize\raggedright
$^\dagger$In the penalty variant, no-bonus dispersion is functionally zero due to the freeze, making the ratio undefined. The aligned interaction's partial $\omega^2$ of $0.144$ is substantially lower than the $0.75$--$0.91$ seen elsewhere.
\end{table*}

\paragraph{On the aligned variant, E3B is actively harmful.}
On the aligned variant, E3B drops average success from $0.486$ to $0.386$ across all five memory architectures. When there is no misranking to correct (potentially dense reward), exploration pressure pulls the policy away from a well-shaped objective. NovelD, pushing only $0.3\times$ the per-step penalty scale, increases average success slightly to $0.512$. E3B's larger magnitude ($2.6\times$) aggressively disrupts the well-shaped dense objective.

\paragraph{Sample efficiency: the bonus does more than raise the ceiling.}
Table~\ref{tab:samples} tracks the environment steps required to first \emph{sustain} an absolute success threshold (mean of five consecutive evaluations).

\begin{table*}[htbp]
\centering\small
\caption{Environment steps (millions) to first \emph{sustain} the threshold (median over successful seeds). $(k/n)$ indicates successful seeds. Thresholds are set well above the no-bonus plateau: $0.30$ (MysteryPath), $0.75$ (MemoryS13), and $0.50$ (TinyReproduce).}
\label{tab:samples}
\begin{tabular}{@{}llccc@{}}
\toprule
Slice & Arch & none & E3B & NovelD \\
\midrule
MPG           & GRU           & never (0/10)  & 16.8 (1/10)  & 15.6 (2/10) \\
sparse        & LSTM          & never (0/10)  & never (0/10) & 14.2 (3/10) \\
$\theta{=}0.30$ & RetNet      & 19.4 (1/10)   & \textbf{4.7} (10/10) & 5.7 (10/10) \\
              & GatedDeltaNet & never (0/10)  & \textbf{5.8} (10/10) & 4.4 (10/10) \\
              & Mamba-2       & 17.6 (8/10)   & 11.6 (10/10) & 10.1 (10/10) \\
\midrule
S13           & GRU           & never (0/10)  & 6.7 (10/10)  & 5.1 (9/10) \\
sparse        & LSTM          & never (0/10)  & 7.9 (10/10)  & 3.2 (10/10) \\
$\theta{=}0.75$ & RetNet      & never (0/10)  & 4.3 (10/10)  & 1.6 (10/10) \\
              & GatedDeltaNet & 12.6 (8/10)   & \textbf{1.5} (10/10) & \textbf{0.7} (10/10) \\
              & Mamba-2       & 5.3 (9/10)    & \textbf{1.3} (10/10) & \textbf{0.4} (10/10) \\
\midrule
Tiny          & GRU           & 2.2 (10/10)   & 2.1 (10/10)  & 2.0 (10/10) \\
sparse        & LSTM          & 1.5 (10/10)   & 1.4 (10/10)  & 1.5 (10/10) \\
$\theta{=}0.50$ & RetNet      & never & never & never \\
              & GatedDeltaNet & 9.5 (2/10)    & 6.9 (1/10)   & 6.9 (3/10) \\
              & Mamba-2       & never & never & never \\
\midrule
Tiny          & GRU           & 1.8 (10/10)   & 1.9 (10/10)  & 1.9 (10/10) \\
dense         & LSTM          & 1.2 (10/10)   & 1.2 (10/10)  & 1.4 (10/10) \\
$\theta{=}0.50$ & RetNet      & never & never & never \\
              & GatedDeltaNet & 8.5 (1/10)    & 9.2 (3/10)   & 7.7 (2/10) \\
              & Mamba-2       & never & never & never \\
\bottomrule
\end{tabular}
\end{table*}

On MysteryPath, the bonus changes the outcome from \emph{never} to \emph{reliably} sustaining success. On MemoryS13, GatedDeltaNet accelerates from 12.6M steps to 0.7M (NovelD), an $18\times$ speed-up. Mamba-2's $0.48$ probability of improvement under E3B masks the fact that it hits final performance a full order of magnitude faster. Equalization eliminates a search problem all cells struggle with. On TinyReproduce, the null pattern holds on the time axis as well.

\paragraph{The memoryless control as a falsification test.}
A policy lacking recurrent state cannot retain exploration findings. A coverage-based bonus should not help it.

\begin{table*}[htbp]
\centering\small
\caption{Memoryless control (mean $\pm$ std over $n{=}10$ seeds) and probability of improvement (PoI). On sparse MysteryPath, the bonus actively degrades the memoryless policy.}
\label{tab:memoryless}
\begin{tabular}{@{}lccccc@{}}
\toprule
Slice & none & E3B & PoI & NovelD & PoI \\
\midrule
MPG / sparse      & $0.062 \pm 0.010$ & $0.002 \pm 0.003$ & \textbf{0.00}
                  & $0.041 \pm 0.008$ & \textbf{0.02} \\
MPG / penalty     & $0.000 \pm 0.000$ & $0.001 \pm 0.001$ & 0.75
                  & $0.014 \pm 0.004$ & 1.00 \\
MPG / distractor  & $0.010 \pm 0.022$ & $0.001 \pm 0.001$ & 0.64
                  & $0.038 \pm 0.007$ & 0.82 \\
S13 / sparse $3\times3$ & $0.352 \pm 0.243$ & $0.498 \pm 0.011$ & 0.54
                  & $0.503 \pm 0.008$ & 0.65 \\
Tiny / sparse     & $0.000 \pm 0.000$ & $0.000 \pm 0.001$ & 0.55
                  & $0.000 \pm 0.000$ & 0.50 \\
Tiny / dense      & $0.000 \pm 0.000$ & $0.000 \pm 0.000$ & 0.50
                  & $0.000 \pm 0.000$ & 0.50 \\
\bottomrule
\end{tabular}
\end{table*}

On MysteryPath sparse, E3B crushes the memoryless control from $0.062$ to $0.002$ (PoI = $0.00$). NovelD yields $0.02$. On MemoryS13, the bonus boosts the memoryless control to the cue-blind chance rate of $0.5$, preventing collapse below chance but failing to help acquire the cue. MysteryPath serves as the definitive falsification test.

\paragraph{Robust aggregation.}
Table~\ref{tab:iqm} displays the interquartile mean (IQM) using a stratified bootstrap \citep{agarwal2021deep}, resistant to bimodal outlier seeds.

\begin{table*}[htbp]
\centering\small
\caption{Interquartile mean (IQM) across all five memory architectures ($50$ runs per entry) with 95\% stratified-bootstrap CIs. All three patterns remain solid under this robust aggregate.}
\label{tab:iqm}
\begin{tabular}{@{}lccc@{}}
\toprule
Slice & none & E3B & NovelD \\
\midrule
MPG / sparse            & 0.154 [0.149, 0.160] & 0.358 [0.345, 0.371] & 0.336 [0.328, 0.344] \\
MPG / aligned           & 0.459 [0.441, 0.478] & 0.354 [0.344, 0.367] & 0.489 [0.475, 0.502] \\
S13 / sparse $3\times3$ & 0.546 [0.520, 0.581] & 0.963 [0.941, 0.974] & 0.953 [0.919, 0.969] \\
Tiny / sparse           & 0.402 [0.382, 0.419] & 0.397 [0.376, 0.421] & 0.432 [0.403, 0.468] \\
Tiny / dense            & 0.411 [0.390, 0.428] & 0.399 [0.375, 0.424] & 0.421 [0.406, 0.440] \\
\bottomrule
\end{tabular}
\end{table*}

Amplification on MysteryPath and equalization on MemoryS13 hold up perfectly, with non-overlapping confidence intervals. TinyReproduce intervals overlap significantly, confirming the null.

\paragraph{Raw per-seed data.}
Table~\ref{tab:seedvectors} presents complete, sorted per-seed tail means. The bimodal split of MemoryS13 is evident: GRU runs without a bonus range from $.48$ to $.52$, while GRU runs with E3B stretch from $.75$ to $1.00$.

\begin{table*}[htbp]\centering\small
\caption{Sorted per-seed tail-mean success ($n{=}10$).}
\label{tab:seedvectors}
\begin{tabular}{@{}lll@{}}\toprule
Arch & Bonus & Per-seed tail-mean success (sorted) \\
\midrule
\multicolumn{3}{@{}l}{\emph{MPG / sparse}} \\
GRU & none & \texttt{.12, .12, .13, .13, .14, .14, .14, .15, .15, .16} \\
GRU & e3b & \texttt{.14, .17, .17, .18, .19, .19, .19, .19, .20, .21} \\
GRU & nvld & \texttt{.15, .17, .17, .19, .19, .19, .19, .19, .20, .21} \\
LSTM & none & \texttt{.13, .13, .14, .14, .14, .14, .14, .15, .15, .15} \\
LSTM & e3b & \texttt{.15, .15, .16, .17, .17, .17, .18, .18, .19, .21} \\
LSTM & nvld & \texttt{.17, .17, .18, .18, .19, .20, .20, .20, .21, .22} \\
RetNet & none & \texttt{.12, .13, .13, .15, .16, .16, .16, .17, .17, .23} \\
RetNet & e3b & \texttt{.48, .49, .53, .55, .55, .57, .57, .60, .62, .65} \\
RetNet & nvld & \texttt{.50, .51, .51, .51, .52, .54, .55, .55, .56, .60} \\
GatedDeltaNet & none & \texttt{.14, .14, .15, .16, .16, .17, .18, .18, .18, .19} \\
GatedDeltaNet & e3b & \texttt{.57, .59, .59, .61, .62, .62, .63, .64, .66, .71} \\
GatedDeltaNet & nvld & \texttt{.49, .53, .54, .57, .58, .61, .62, .63, .65, .66} \\
Mamba-2 & none & \texttt{.15, .18, .20, .20, .20, .21, .22, .22, .22, .23} \\
Mamba-2 & e3b & \texttt{.26, .28, .31, .35, .35, .36, .36, .36, .38, .41} \\
Mamba-2 & nvld & \texttt{.24, .29, .29, .30, .30, .31, .31, .32, .32, .33} \\
Memoryless & none & \texttt{.05, .05, .06, .06, .06, .06, .06, .07, .07, .08} \\
Memoryless & e3b & \texttt{.00, .00, .00, .00, .00, .00, .00, .00, .00, .01} \\
Memoryless & nvld & \texttt{.03, .03, .03, .04, .04, .04, .04, .04, .05, .05} \\
\midrule
\multicolumn{3}{@{}l}{\emph{S13 / sparse $3\times3$}} \\
GRU & none & \texttt{.48, .49, .49, .50, .50, .50, .51, .51, .51, .52} \\
GRU & e3b & \texttt{.75, .81, .87, .92, .95, .96, .98, .98, 1.00, 1.00} \\
GRU & nvld & \texttt{.52, .54, .74, .92, .94, .98, .98, .98, .98, 1.00} \\
LSTM & none & \texttt{.48, .48, .49, .49, .50, .50, .51, .51, .52, .53} \\
LSTM & e3b & \texttt{.95, .96, .96, .98, .98, .99, 1.00, 1.00, 1.00, 1.00} \\
LSTM & nvld & \texttt{.69, .90, .95, .95, .98, .99, .99, 1.00, 1.00, 1.00} \\
RetNet & none & \texttt{.44, .48, .49, .49, .49, .49, .50, .50, .50, .51} \\
RetNet & e3b & \texttt{.68, .76, .91, .94, .94, .95, .98, .98, .99, 1.00} \\
RetNet & nvld & \texttt{.53, .56, .62, .63, .73, .92, .94, .97, .98, .99} \\
GatedDeltaNet & none & \texttt{.52, .52, .62, .71, .75, .77, .77, .92, .99, .99} \\
GatedDeltaNet & e3b & \texttt{.88, .96, .97, .98, .98, .99, .99, 1.00, 1.00, 1.00} \\
GatedDeltaNet & nvld & \texttt{.92, .94, .99, .99, .99, .99, .99, 1.00, 1.00, 1.00} \\
Mamba-2 & none & \texttt{.55, .63, .85, .90, .92, .92, .93, .94, .96, .96} \\
Mamba-2 & e3b & \texttt{.38, .48, .51, .65, .79, .94, .94, .95, .97, .97} \\
Mamba-2 & nvld & \texttt{.80, .80, .88, .91, .91, .93, .94, .95, .96, .99} \\
Memoryless & none & \texttt{.00, .00, .00, .49, .50, .50, .50, .51, .51, .51} \\
Memoryless & e3b & \texttt{.48, .49, .49, .49, .49, .50, .50, .51, .51, .51} \\
Memoryless & nvld & \texttt{.49, .50, .50, .50, .50, .50, .51, .51, .51, .51} \\
\bottomrule\end{tabular}
\end{table*}

\section{Does the Bonus Change What Memory Encodes?}
\label{app:probe}

A logical mechanistic assumption is that the episodic bonus works by upgrading the recurrent state: pushing exploration provides a richer training signal, making the task-relevant latent more decodable, leading to better behavior. This section tests that hypothesis. \emph{It fails.} This negative mechanistic result contextualizes the positive behavioral results.

\paragraph{Pre-registration.}
We established six predictions (P1--P6) prior to fitting probes. Table~\ref{tab:probe-preregistered} grades the outcomes. 

\paragraph{Protocol.}
For each cell, we replay a fixed set of 200 evaluation episodes on the final checkpoint, recording the recurrent state $h_t$. We fit a logistic probe on $h_t$ to predict a task-relevant latent, tuning $\ell_2$ strength on a validation split. Data is split by \emph{episode} (70/15/15) to prevent leakage. For MysteryPath, the latent is ``is this adjacent, tested tile on-path?''. For MemoryS13, it is cue identity. Scope: 3 architectures $\times$ \{none, E3B\} $\times$ 3 (env, reward) slices, testing 2 seeds per cell.

\paragraph{Three controls, one of which fails informatively.}
We ran a shuffled-label null, an observation-only probe $\psi(o_t)$, and a randomly initialized cell. On MysteryPath, $\psi(o_t)$ scores $0.58$ because the environment visually flags tested tiles. The memory-specific value is $h_t$ relative to this baseline. On MemoryS13, an untrained cell decodes cue identity at $0.94$--$1.00$ due to classic reservoir behavior, rendering P4 untestable.

\begin{table}[htbp]\centering\small
\caption{Control check. The shuffled-label null passes. The other two controls fail informatively on one environment each.}
\label{tab:probe-controls}
\begin{tabular}{@{}lccl@{}}\toprule
Control & MysteryPath & MemoryS13 & Verdict \\
\midrule
shuffled label      & 0.48--0.51 & ${\approx}0.50$ & pass \\
obs-only $\psi(o_t)$ & \textbf{0.58} & 0.53 & MPG fails \\
random-init cell    & 0.55--0.60 & \textbf{0.94--1.00} & S13 fails \\
\bottomrule\end{tabular}
\vspace{2pt}\par\footnotesize\raggedright
Of 36 shuffled-label fits, only three stray from $0.5$ by more than $0.08$; all are \emph{below} chance, pointing to small-sample noise rather than leakage.
\end{table}

\paragraph{Result: the bonus does not improve encoding.}

\begin{table}[htbp]\centering\small
\caption{Impact of E3B on decodability ($\Delta\mathrm{AUC}$) with 95\% CIs. No cell shows the expected increase; three show a significant \emph{decrease}.}
\label{tab:probe-delta}
\begin{tabular}{@{}llcc@{}}\toprule
Variant & Arch & $\Delta$AUC & 95\% CI \\
\midrule
MPG / sparse  & GRU           & $-0.031$ & $[-0.034, -0.028]$ \\
MPG / sparse  & LSTM          & $-0.051$ & $[-0.062, -0.040]$ \\
MPG / sparse  & GatedDeltaNet & $+0.016$ & $[-0.001, +0.033]$ \\
\midrule
MPG / aligned & GRU           & $+0.003$ & $[-0.004, +0.009]$ \\
MPG / aligned & LSTM          & $+0.001$ & $[-0.009, +0.011]$ \\
MPG / aligned & GatedDeltaNet & $-0.020$ & $[-0.027, -0.014]$ \\
\midrule
S13 / sparse  & GRU           & $+0.048$ & $[-0.031, +0.126]$ \\
S13 / sparse  & LSTM          & $+0.086$ & $[-0.001, +0.172]$ \\
S13 / sparse  & GatedDeltaNet & $+0.000$ & $[+0.000, +0.001]$ \\
\bottomrule\end{tabular}
\end{table}

\paragraph{The critical disconnect.}
On MysteryPath sparse, GatedDeltaNet's success jumps from $0.17$ to $0.62$ with E3B (Table~\ref{tab:seedvectors}), but path decodability barely moves ($+0.016$). The bonus does not turn the recurrent state into a better encoder.

\paragraph{The complementary disconnect on MemoryS13.}
Without a bonus, GRU and LSTM decode cue identity at $0.928$ and $0.876$ but fail the task on every seed (Table~\ref{tab:s13binary}). Forming the memory is not the bottleneck; using it is. The penalty arm corroborates this: the frozen agent never reaches the goal (Table~\ref{tab:firstgoal}), proving the failure lies in sampling, not encoding. The bonus provides behavioral coverage, not representational quality.

\begin{figure*}[htbp]\centering
\includegraphics[width=\linewidth]{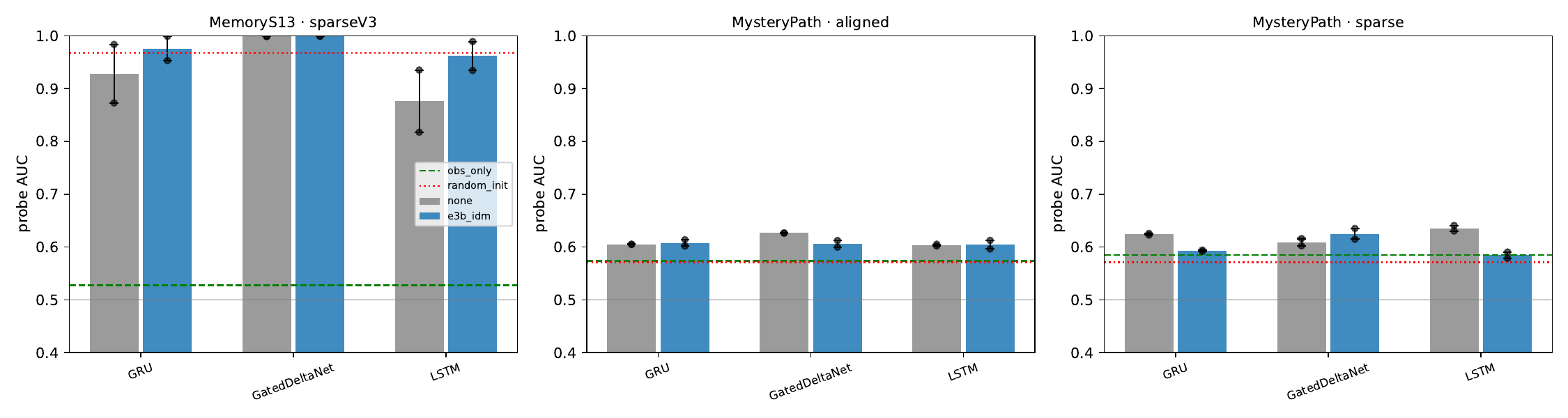}
\caption{Probe AUC by architecture, grouped by training signal. Dashed lines mark obs-only and random-init controls. MemoryS13 random-init is saturated.}
\label{fig:probe-f1}
\end{figure*}

\begin{figure}[htbp]\centering
\includegraphics[width=\linewidth]{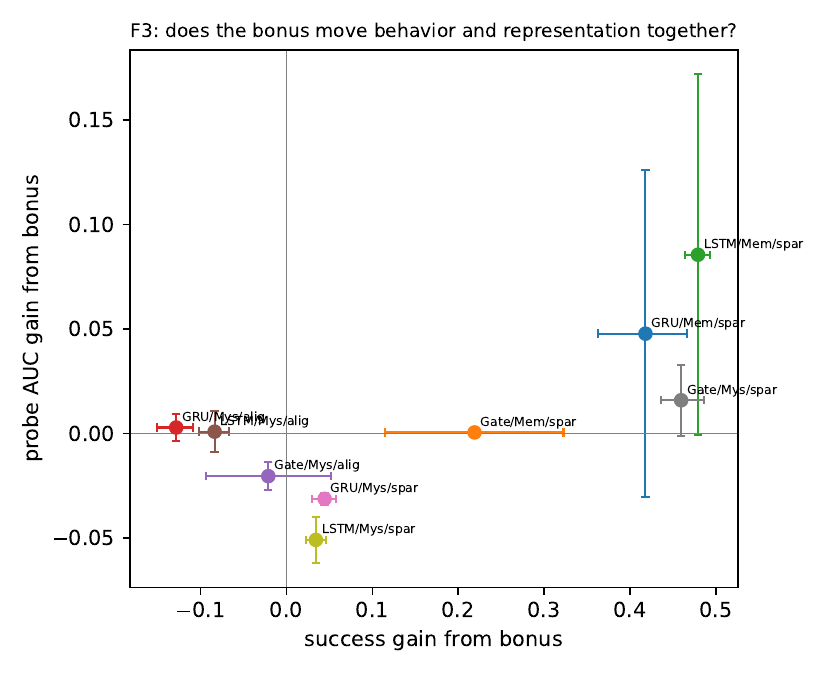}
\caption{Success gain vs.\ decodability gain under E3B. The largest behavioral gains occur with near-zero decodability gain.}
\label{fig:probe-f3}
\end{figure}

\begin{figure}[htbp]\centering
\includegraphics[width=\linewidth]{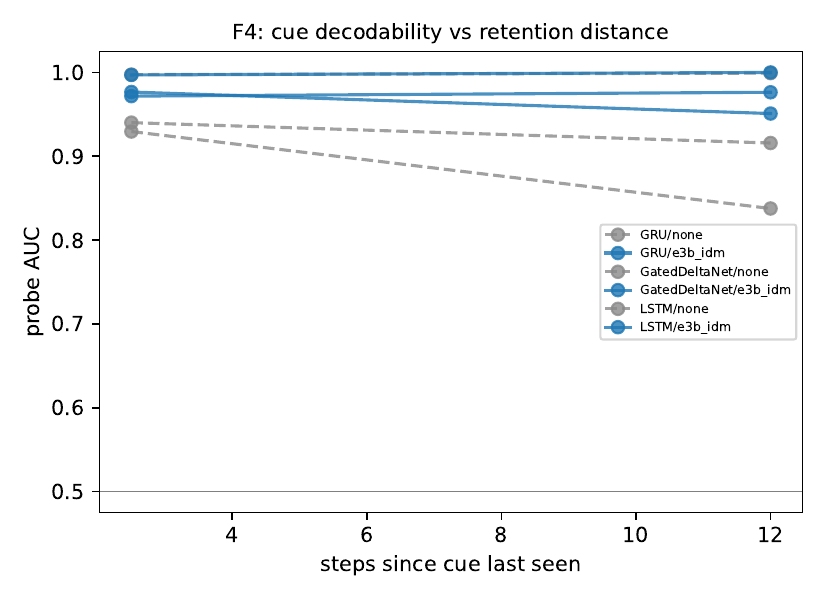}
\caption{MemoryS13 probe AUC vs.\ steps since cue was visible. Decodability remains high during traverse for all cells. Retention is not the separator.}
\label{fig:probe-f4}
\end{figure}

\begin{table*}[htbp]\centering\small
\caption{Pre-registered predictions vs.\ actual outcomes. Two misses, two hits, one split, two untestable.}
\label{tab:probe-preregistered}
\begin{tabular}{@{}lp{4.9cm}p{4.4cm}@{}}\toprule
ID & Prediction & Outcome \\
\midrule
P1 & MPG/sparse: bonus increases path decodability
   & \textbf{Miss.} Decreased for GRU/LSTM, flat for GDN. \\
P2 & Increase is larger for GDN than GRU/LSTM
   & \textbf{Hit on ordering only.} GDN $>$ GRU $>$ LSTM, but no actual increase. \\
P3 & MPG/aligned: bonus does \emph{not} increase decodability
   & \textbf{Hit.} All effects ${\approx}0$ or negative. \\
P4 & S13: bonus increases cue decodability
   & \textbf{Not evaluable.} Probe saturated by reservoir dynamics. \\
P5 & Obs-only control sits at chance
   & \textbf{Split.} Hit on S13 ($0.53$); missed on MPG ($0.58$). \\
P6 & Probe AUC correlates with success
   & \textbf{Not trustworthy.} Correlation is an artifact of between-environment differences. \\
\bottomrule\end{tabular}
\end{table*}

\paragraph{Limitations.}
Two seeds per cell yields directional estimates. The MysteryPath obs-only control inflates the baseline. The MemoryS13 probe is saturated. Linear probes establish only a lower bound on decodability. However, the central disconnect remains: massive behavioral shifts occur alongside near-zero representational ones.

\section{Summary Tables}
\label{app:summary}

\begin{table}[htbp]\centering\small
\caption{Evaluated architectures by memory mechanism class.}
\label{tab:arch}
\begin{tabular}{@{}lll@{}}
\toprule
Architecture & Mechanism & Known characteristics \\
\midrule
GRU & gated recurrence & robust; limited capacity \\
LSTM & gated recurrence & robust; limited capacity \\
RetNet & retention (decay) & high capacity; copying gap \\
GatedDeltaNet & gated linear attention & high capacity \\
Mamba-2 & selective SSM & high capacity; copying gap \\
Memoryless & none (control) & lower bound \\
\bottomrule
\end{tabular}
\end{table}

\paragraph{Why we left out Transformers.}
We limited the lineup to architectures maintaining a fixed-size recurrent state updated online. A growing attention context fundamentally alters the training stack. Given attention's copying advantage \citep{jelassi2024repeat}, Transformers would likely dominate TinyReproduce. Evaluating them fairly remains future work.

\section{Additional Notes for the Careful Reader}
\label{app:faq}

\paragraph{Why does sparse TinyReproduce show a null result despite potential sparsity?}
The observation stream operates on a fixed schedule. Intrinsic and extrinsic returns rank policies identically, leaving the bonus with no power to reorder them (redundancy supersedes sparsity).

\paragraph{Why do some architectures perform above chance on MemoryS13 without a bonus?}
GatedDeltaNet and Mamba-2 escape the cue-blind optimum in most seeds via entropy-driven exploration. The bimodal seed distribution (Fig.~\ref{fig:seed-strip}) suggests this relies on lucky early trajectories. The bonus removes this lottery for failing cells but does little for those already succeeding.

\paragraph{Is the TinyReproduce failure for RetNet/Mamba-2 a tuning issue?}
The failure persists across tested learning rates ($10^{-4}$ and $10^{-3}$) and mirrors documented copying gaps \citep{jelassi2024repeat}. The null bonus effect held across all configurations.

\paragraph{Does the memoryless decline under E3B directly support the theory?}
It aligns perfectly: E3B redirects the policy toward novel states, but without memory, the agent cannot convert this into return and loses its baseline unbiased-wandering return. The theory describes what a bonus \emph{can} change, not the optimization dynamics of an incapable policy.

\section{Compute and Reproducibility}
\label{app:compute}

Runs operated as single-GPU jobs across a heterogeneous cluster.

\begin{table}[htbp]\centering\small
\caption{Measured compute. Median wall-clock time across finished runs; GPU-hours = median time $\times$ number of runs. The 1{,}800 runs match the released per-seed records.}
\label{tab:compute}
\begin{tabular}{@{}lcccc@{}}
\toprule
Arm & Budget & Runs & Median & GPU-h \\
\midrule
MPG sparse / penalty / aligned & 20M & 540 & 13.6\,h & 7{,}344 \\
MPG distractor                 & 20M & 180 & 11.9\,h & 2{,}142 \\
S13 sparse $3{\times}3$ / $7{\times}7$ / penalty & 20M & 540 & 8.5\,h & 4{,}590 \\
S13 distractor                 & 20M & 180 & 7.4\,h  & 1{,}332 \\
TinyReproduce dense / sparse   & 10M & 360 & 3.9\,h  & 1{,}404 \\
\midrule
\textbf{Reported matrix}       &     & \textbf{1{,}800} & & \textbf{16{,}812} \\
\bottomrule
\end{tabular}
\vspace{2pt}\par\footnotesize\raggedright
Total project cost (including pilots and calibrations) was roughly 21{,}200 GPU-hours.
\end{table}

Code, per-seed tail means, evaluation curves, and analysis scripts will be released upon publication.

\section{Formal Results and Proofs}
\label{app:proofs}

\subsection{Encoders Need Not Exist in General}

Encoding is demanding even in expectation. The following construction shows that
a Markovian reward over three states can require more machine states than any
finite automaton has.

\begin{proposition}
\label{prop:horizon}
There is a POMDP $\mathcal{M}$ with $|S| = 3$, $|O| = 3$, $|A| = 3$ and a
Markovian reward, such that (i) in unbounded horizon, no finite RM over
observations matches its expected return for all observation-history policies;
and (ii) under any horizon $H$, \emph{every} RM over observations that does so
has at least $H - 1$ states.
\end{proposition}

\begin{proof}
Let $\mathcal{M}$ have states $S = \{s_0, s_1, s_d\}$, actions
$A = \{c, g_0, g_1\}$, and observations $O = \{o^0, o^1, o^\bot\}$. The start
state is $s_0$ or $s_1$ with equal probability. Action $c$ keeps the state fixed,
while $g_0$ and $g_1$ move to the absorbing state $s_d$. In state $s_z$ the
observation equals $o^z$ with probability $q$ and the other symbol with
probability $1-q$, for a fixed $\tfrac12 < q < 1$, and $s_d$ always shows
$o^\bot$. The reward of guess $g_i$ is $1$ if the state was $s_i$ and $0$
otherwise, and all other rewards are $0$. In plain terms, a hidden coin, noisy
peeks, and a guess that ends the episode. By construction the reward is
Markovian.

For any policy whose first $k-1$ steps choose action $c$,
\begin{align*}
\Pr(s_1 \mid o_0 = \dots = o_{k-1} = o^1)
&= \frac{\frac12 q^k}{\frac12 q^k + \frac12 (1-q)^k}
= \sigma(k\lambda),
\end{align*}
where $\sigma(x) := \frac{e^x}{1+e^x}$ and
$\lambda := \ln\frac{q}{1-q} > 0$. Since
$\sigma'(x) = \sigma(x)(1 - \sigma(x)) > 0$, each additional matching peek makes
the bettor strictly more confident, without bound in $k$.

Suppose a POMDPRM $\mathcal{T}$ with
$\mathcal{R}_{POA} = \tuple{U, u_0, \delta_u, \delta_r}$ encodes $\mathcal{M}$.
Consider the two policies $\pi_k^0$ and $\pi_k^1$ that take $c$ for the first
$k-1$ steps and then take $g_0$ and $g_1$ respectively, and consider the trace on
which all observations of the first $k-1$ steps equal $o^1$. Let
$x_0, x_1, x_2, \dots$ be the machine states along that trace, with $x_0 = u_0$
by definition. For $k = 1$ the expected returns of $\pi_1^0$ and $\pi_1^1$ are
$1 - \sigma(\lambda)$ and $\sigma(\lambda)$, so matching them forces
$\delta_r\big(x_0, L(o^1, g_0, o^\bot)\big) = 1 - \sigma(\lambda)$ and
$\delta_r\big(x_0, L(o^1, g_1, o^\bot)\big) = \sigma(\lambda)$.

We claim that $x_0, \dots, x_{k-1}$ are pairwise distinct with
$\delta_r\big(x_{k-1}, L(o^1, g_i, o^\bot)\big)$ equal to $1 - \sigma(k\lambda)$
for $i = 0$ and $\sigma(k\lambda)$ for $i = 1$, by induction on $k$. Suppose it
holds up to $k$. First, $\delta_r\big(x_j, L(o^1, c, o^1)\big) = 0$ for
$0 \le j \le k-1$, as otherwise the expected returns of $\pi_j^0, \pi_j^1$ would
not match. Now $x_k = \delta_u\big(x_{k-1}, L(o^1, c, o^1)\big)$. If $x_k = x_j$
for some $j \le k-1$, the guess rewards at $x_k$ and $x_j$ coincide; but the
returns of $\pi_{k+1}^0, \pi_{k+1}^1$ pin the former to
$1 - \sigma\big((k{+}1)\lambda\big), \sigma\big((k{+}1)\lambda\big)$ while those
of $\pi_{j+1}^0, \pi_{j+1}^1$ pin the latter to
$1 - \sigma\big((j{+}1)\lambda\big), \sigma\big((j{+}1)\lambda\big)$, and
$\sigma$ is strictly increasing, a contradiction.

For any $k + 1 > |U|$ the trace exhibits $k$ distinct machine states, so by the
pigeonhole principle no finite RM encodes the reward at unbounded horizon. If the
maximum horizon is $H$ and the RM encodes the reward, the states
$x_0, \dots, x_{H-2}$ are pairwise distinct, so $|U| \ge H - 1$.
\end{proof}

\subsection{The Belief Filter Is a Perfect Encoder}

The impossibility above disappears when the observations leave only finitely many
posteriors over the hidden pair. In that case an encoding machine exists, is
computable, and is perfect.

\begin{proposition}[Belief filter]
\label{prop:filter}
Let $\mathcal{M}$ be a POMDP with $\gamma \in (0,1)$ such that some MDPRM
$\tuple{S, A, p, \gamma, \mu, P', L', \mathcal{R}_{P'SA}}$, with
$\mathcal{R}_{P'SA} = \tuple{U', u_0', \delta_u', \delta_r'}$, encodes
$\mathcal{M}$, and suppose the set of posteriors over $Z := S \times U'$
reachable at positive probability is finite. Then the belief filter over $Z$,
whose reward on each transition is the average of the latent reward given the
observable past, is an explicitly computable POMDPRM that encodes $\mathcal{M}$
and is perfect w.r.t.\ its labelling.
\end{proposition}

\begin{proof}
Write $z_t := (s_t, u_t)$ for the pair of the hidden state and the latent machine
state, where $u_{t+1} = \delta_u'\big(u_t, L'(s_t, a_t, s_{t+1})\big)$. The MDPRM
hypothesis makes $z_t$ a Markov chain and makes
$r_t = \delta_r'\big(u_t, L'(s_t, a_t, s_{t+1})\big)$ a fixed function of
$(z_t, a_t, s_{t+1})$ on every trajectory.

The machine tracks the Bayes posterior over this pair. For an observable history
$h$ of positive probability, let $b_h(z)$ be the posterior probability of
$z_t = z$ given $h$. This posterior does not depend on the policy that produced
$h$, because the probability of any latent path together with $h$ factors into
environment terms times the policy terms $\prod_j \pi_j(a_j \mid h_j)$, the
policy terms depend on $h$ alone, and they cancel between the numerator and the
denominator of the conditional. By the same cancellation, the predicted
probability of the next observation, $\Pr(o_{t+1} = o' \mid h_t, a_t)$, is a
fixed function $N(b_{h_t}, a_t, o')$ of the current posterior, obtained by
propagating $b_{h_t}$ through $p$ and $\omega$; and one step of Bayes updating
carries $b_h$ to a posterior $F(b_h, a, o')$ that again depends only on
$(b_h, a, o')$, obtained by propagating $b_h$ through $p$ and the latent machine
update, weighting by $\omega(o' \mid s')$, and renormalizing by $N$. Consequently
every reachable posterior arises from a first-step posterior (the same formula
with $\mu$ in place of $b$) by iterating $F$, and by hypothesis the set
$\widehat{B}$ of reachable posteriors is finite.

The machine is the filter itself. Choose $P$ so that some injection
$\mathrm{code} : O \times A \times O \to 2^{P}$ exists. The states are the initial $\iota$, the posteriors in
$\widehat{B}$, and a sink $\bot$ for updates of probability zero, which loops
with reward $0$; the transitions apply $F$; and the reward on a transition is the
posterior mean of the latent reward,
\begin{align*}
&\delta_r\big(b, L(o, a, o')\big)\\
  &= \frac{\sum_{(s,u) \in Z} \sum_{s' \in S}
  b(s,u)\, p(s' \mid s, a)
  \omega(o' \mid s')\,
  \delta_r'\big(u, L'(s, a, s')\big)}{N(b, a, o')}
\end{align*}
with $\mu$ in place of $b$ out of $\iota$. Every table is computable in exact
arithmetic by forward closure from the first-step posteriors, terminating because
$\widehat{B}$ is finite.

The machine encodes. By induction on the update rule, on every history of
positive probability the machine state equals the posterior, $x_t = b_{h_t}$, so
its reward is $\tilde r_t = \mathbb{E}\big[r_t \mid h_t, a_t, o_{t+1}\big]$.
Taking total expectations over the countably many triples $(h_t, a_t, o_{t+1})$
gives $\mathbb{E}^\pi[\tilde r_t] = \mathbb{E}^\pi[r_t]$ for every $t$ and
policy; all rewards are bounded by the largest entry of $\delta_r'$, so both
discounted sums converge absolutely and the returns agree.

The machine is perfect. Given $h_t$ and $a$, the joint distribution of
$(o_{t+1}, r_t)$ is obtained by propagating $b_{h_t}$ one step through $p$,
$\omega$, and the latent reward table, so it is a fixed function of
$(b_{h_t}, a) = (x_t, a)$. This is Eq.\ (3) of the main text.
\end{proof}

\subsection{The Reference Kernel of the Greedy RM Policy}

\begin{remark}[The reference kernel, and a caution about coarse machines]
\label{rem:generalkernel}
Definitions 4 and 5 of the main text apply to any POMDPRM $\mathcal{T}$ encoding
$\mathcal{M}$, perfect or not. Fix the uniform reference policy $\pi_0$ and let
$d(o, u) = \sum_{t \ge 0} \gamma^t \Pr^{\pi_0}(o_t = o, x_t = u)$ be the
discounted occupancy of the pair $(o, u)$. For pairs with $d(o, u) > 0$, let
$\Pr_{\pi_0}(o' \mid o, u, a)$ be the weighted average of the history
conditionals $\Pr(o_{t+1} = o' \mid h_t, a)$ over all times $t$ and histories
$h_t$ with $(o_t, x_t) = (o, u)$, each weighted by
$\gamma^t \Pr^{\pi_0}(h_t) / d(o, u)$; for $d(o, u) = 0$ let the kernel be
uniform. The history conditional does not depend on the policy, by the
cancellation argument of the belief-filter proof applied to the plain POMDP, so
the kernel is well defined and depends only on the reference policy and
$\gamma$.

When $\mathcal{T}$ is perfect, $\Pr(o_{t+1} = o' \mid h_t, a)$ is constant across
the histories in each $(o, u)$ cell (sum Eq.\ (3) of the main text over reward
values), so the kernel equals that common value. It is then independent of the
reference policy, the discounting, and $t$, and coincides with the canonical
$\Pr(o_{t+1} \mid o_t, x_t, a_t)$. Since the full-support $\pi_0$ reaches every
pair that any policy reaches, positive occupancy coincides with reachability.

A caution accompanies the generality. A machine too coarse cannot register a
misranking its state does not represent. On the one-state MysteryPath machine of
Proposition~\ref{prop:nonperfect}, both actions induce the same successor
distribution from every pair, so every $\phi$-greedy policy ties everywhere and
$r$ comes out potentially dense w.r.t.\ it, even though this reward undervalues
exactly the probing actions. Potential density w.r.t.\ a non-perfect encoder
therefore certifies nothing, and the classification only reads it where the
reward is structurally dense, where every encoder is perfect and the kernel is
canonical.
\end{remark}

\subsection{Classification of the MysteryPath Abstraction}

The recipe from the main text asks for one machine that forgets and still
encodes. The one-state goal machine is that witness, and the two histories ``went
left, fell'' and ``went right, fell'' show it is not perfect.

\begin{proposition}[Encoding does not imply perfection; the MysteryPath
abstraction is structurally sparse]
\label{prop:nonperfect}
There is a POMDP $\mathcal{M}$ (a one-cell abstraction of MysteryPath) and a
POMDPRM $\mathcal{T}$ encoding $\mathcal{M}$ whose RM has a single non-terminal
state, reproduces the reward step by step, and is not perfect w.r.t.\ its
labelling; moreover no POMDPRM with a single non-terminal state that encodes
$\mathcal{M}$ is perfect. By Definition 3 of the main text, $r$ is structurally
sparse.
\end{proposition}

\begin{proof}
A hidden trap side $z \in \{L, R\}$ is drawn uniformly. The states are
$(\mathrm{start}, z)$ with observation $o_S$, $(\mathrm{fallen}, z)$ with
observation $o_F$, $\mathrm{done}_1$ with observation $o_G$, and $\mathrm{done}_2$
with observation $o_\bot$, and the actions are $\{a_L, a_R\}$. From either
start-type or fallen-type state with hidden $z$, action $a_z$ leads to
$(\mathrm{fallen}, z)$ with reward $0$, and action $a_{\bar z}$ leads to
$\mathrm{done}_1$ with reward $+1$. $\mathrm{done}_1$ moves to $\mathrm{done}_2$
under any action, $\mathrm{done}_2$ is absorbing, and all rewards from the done
states are $0$.

The machine $\mathcal{T}$ has one non-terminal state $u$, with
$P = \{\mathsf G\}$, $L(o,a,o') = \{\mathsf G\}$ iff $o' = o_G$,
$\delta_u \equiv u$, $\delta_r(u, \{\mathsf G\}) = 1$, and
$\delta_r(u, \varnothing) = 0$. The environment gives $+1$ exactly on transitions
into $\mathrm{done}_1$, whose arrival observation is $o_G$, and $o_G$ occurs at
no other time, so $\tilde r_t = \mathbbm{1}[o_{t+1} = o_G] = r_t$ on every
trajectory and $\mathcal{T}$ encodes.

Now let $\mathcal{T}'$ be any encoder with a single non-terminal state. Its RM
state $x_t$ is constant, so perfection would force
$\Pr(o_{t+1}, r_t \mid h_t, a_t)$ to depend on $h_t$ only through $o_t$. The
histories $h = (o_S, a_L, o_F)$ and $h' = (o_S, a_R, o_F)$ both have positive
probability under the uniform policy and share
$(o_t, x_t, a_t) = (o_F, u, a_L)$. But $h$ reveals $z = L$, so $a_L$ falls again
with probability $1$, while $h'$ reveals $z = R$, so $a_L$ reaches the goal with
probability $1$. The conditional distributions differ, so no single-state machine
is perfect.
\end{proof}

\subsection{Classification of the MemoryS13 Cue Abstraction}

The raw environment is \emph{not} structurally dense. The cue machine of
Figure 2(b) of the main text encodes the raw environment as well, yet on the raw
views it is not perfect, because inside a corridor whose cells share a view the
next observation depends on the unobserved position, and by Definition 3 one
non-perfect encoder makes the reward structurally sparse. That failure belongs to
the dynamics, since the reward never distinguishes corridor positions. The
abstraction below removes exactly this view-aliasing and nothing else, and the
density claim is about the abstraction. Here the two-history recipe runs in
reverse. Any machine that tries to forget the cue and still match the returns of
policies that differ only in the final pick is caught, since matching pins its
remaining discounted reward at the shared state to two different values at once.

\begin{lemma}[Density certificate for the cue abstraction]
\label{lem:cueperfect}
Let $\mathcal{M}_{\mathrm{cue}}$ be the cue abstraction of MiniGrid-MemoryS13. A
hidden cue $c \in \{A,B\}$ is uniform; the locations are start ($o_S$), cue room
($o_C^A$ or $o_C^B$, revealing $c$), junction ($o_J$), and end ($o_E$, then
absorbing $o_\bot$). At the start, $\mathrm{peek}$ moves to the cue room and
every other action to the junction; every action leaves the cue room for the
junction; at the junction $\mathrm{pick}_i$ moves to the end with reward
$\mathbbm{1}[i = c]$ while other actions self-loop with reward $0$; all other
rewards are $0$. Then every POMDPRM that encodes $\mathcal{M}_{\mathrm{cue}}$ is
perfect w.r.t.\ its labelling, so the abstraction's reward is structurally dense.
The three-state cue machine of Figure 2(b) of the main text is such an encoder,
and it does not reproduce the reward step by step.
\end{lemma}

\begin{proof}
Fix any encoder $\mathcal{T}$. For a positive-probability history $h$ ending at
the junction, let $\kappa(h) \in \{A, B, \varnothing\}$ record whether $h$
contains $o_C^A$, contains $o_C^B$, or contains neither, and let $x(h)$ be the
machine state after reading the labels of $h$.

We first show that $x(h)$ determines $\kappa(h)$. After a pick the future is
deterministic, so the machine's discounted reward from the pick onward is a
function of the state at the pick and of which pick was played; write $G_i(x)$
for this quantity when $\mathrm{pick}_i$ is played from state $x$, and
$\Psi(x) := G_A(x) - G_B(x)$. Compare the two deterministic policies that produce
$h$ along its prefix, act identically at every other history, and at $h$ play
$\mathrm{pick}_A$ and $\mathrm{pick}_B$ respectively. Away from $h$ they generate
identical trajectories and label streams, so both return differences are
supported on the event that $h$ occurs. On that event the environment returns
differ by
$\gamma^{|h|}\Pr(h)\big(\Pr(c{=}A \mid h) - \Pr(c{=}B \mid h)\big)$ while the
machine returns differ by $\gamma^{|h|}\Pr(h)\,\Psi(x(h))$, the shared prefix
cancelling. Encoding equates the two returns for both policies, so
$$\Psi\big(x(h)\big) = \Pr(c{=}A \mid h) - \Pr(c{=}B \mid h),$$
which equals $1$, $-1$, or $0$ according as $\kappa(h)$ is $A$, $B$, or
$\varnothing$. Since $\Psi$ is a function of the state alone, two junction
histories with different classes cannot share a machine state.

Perfection follows cell by cell. In a cell $\{o_t = o, x_t = u\}$ with $o = o_S$
the only history is the empty one; with $o = o_C^X$ every history has revealed
$c = X$; with $o = o_J$ all histories share $\kappa$ by the previous paragraph,
and given $\kappa$ the pick reward is Bernoulli with mean
$\Pr(c = i \mid \kappa)$ while everything else is deterministic; with
$o \in \{o_E, o_\bot\}$ the process is deterministic with reward $0$. In every
case the conditional distribution of $(o_{t+1}, r_t)$ is constant across the
histories in the cell, which is Eq.\ (3) of the main text.

Finally, the cue machine encodes $\mathcal{M}_{\mathrm{cue}}$, since its reward
on every transition is the average environment reward given the history and the
tower property gives equal expected returns for every policy; it does not
reproduce the reward step by step, since from its no-cue state a pick is rewarded
a fixed $\tfrac12$ while the environment gives $0$ or $1$.
\end{proof}

\begin{proof}[Computation (potential sparsity of the cue abstraction)]
Take the cue machine, with rewards $1$ and $0$ for matching and mismatching picks
from $u_A, u_B$ and $\tfrac12$ for either pick from $u_0$. At $(o_S, u_0)$ all
one-step machine rewards are $0$, so the greedy policy mixes uniformly over the
four actions, peeking with probability $\tfrac14$ and reaching the junction
cue-blind with probability $\tfrac34$; at $(o_J, u_0)$ it picks uniformly; at
$(o_J, u_X)$ it plays $\mathrm{pick}_X$; at the cue room all actions give $0$ and
lead to the junction. Hence
$$J(\pi) = \tfrac34 \cdot \gamma \cdot \tfrac12
+ \tfrac14 \cdot \gamma^2 = \tfrac38\gamma + \tfrac14\gamma^2.$$
With $\phi = V^\star$ as in the proof of Theorem 1,
$J(\pi^{V^\star}) = J^\star = \gamma^2$ for $\gamma > \tfrac12$, since peek, go,
pick correctly then beats guessing immediately, worth $\gamma \cdot \tfrac12$,
and
$$J(\pi^{V^\star}) - J(\pi) = \tfrac38\gamma\,(2\gamma - 1) > 0,$$
so the abstraction's reward is potentially sparse w.r.t.\ the cue machine.
\end{proof}

\subsection{Classification of the TinyReproduce Abstraction}

TinyReproduce invites a shortcut. A mistake ends the episode, so one might hope a
machine could simply reward surviving play steps and forget the tokens. The
shortcut is not available, because episode end is a consequence of correctness,
not an input the machine can read. The labelling function sees only
$(o_t, a_t, o_{t+1})$, the play observation is one constant view, and the final
step shows that same view whether the token was right or wrong. Two dictations
that agree on the tokens queried so far but disagree on the current target give,
under the same played token, identical $(o, a, o')$ and hence identical labels
with different rewards. The only way to reward a token correctly is to already
know it, and knowing every token at its query means carrying the dictated
sequence through the play phase, as the sequence machine of Figure 2(c) of the
main text does. The lemma shows every encoder is forced into this shape.

\begin{lemma}[Density certificate for the dictation abstraction]
\label{lem:tinydense}
Consider the dictation abstraction of TinyReproduce matching the implementation.
A sequence $s \in \Sigma^k$ ($k \ge 2$) drawn uniformly is shown token by token
at times $0, \dots, k{-}1$ (the observation at time $j$ is $s_j$, and watch-phase
actions do not affect the state), then queried in reverse. The observation at
every later step is one constant $o_{\mathrm{play}}$; an incorrect token moves
the process to an absorbing state, completion does too, and the absorbing state
also shows $o_{\mathrm{play}}$, so no observation distinguishes a right token, a
wrong token, or the end of the episode. The dense variant gives $+1/k$ on each
correct token; the sparse variant gives $\tau/k$ on the transition into the
absorbing state, where $\tau$ is the number of correct tokens. Then, in the dense
variant, and in the sparse variant whenever $\gamma > (k-1)/k$, every POMDPRM
that encodes $\mathcal{M}$ resolves the dictated sequence, in that its state at
the end of the watch phase determines $s$, and is perfect w.r.t.\ its labelling.
Hence the abstraction's reward is structurally dense in both variants.
\end{lemma}

\begin{proof}
Since the play and absorbing observations coincide, every label from time $k$ on
is $\ell_a = L(o_{\mathrm{play}}, a, o_{\mathrm{play}})$, a function of the
action alone; and since those observations carry no information, a deterministic
policy's actions from time $k$ on form a fixed word $w \in \Sigma^{\omega}$
determined by the dictation it watched. Conversely, every map $s \mapsto w(s)$,
combined with any watch behavior $\alpha$, is a deterministic policy, and every
dictation has positive probability.

Fix an encoder $\mathcal{T}$. For a machine state $\xi$ and a word $w$, let
$G(\xi, w) = \sum_{j \ge 0} \gamma^j\, \delta_r(\xi_j, \ell_{w_j})$ with
$\xi_0 = \xi$ and $\xi_{j+1} = \delta_u(\xi_j, \ell_{w_j})$, the machine's
discounted reward from time $k$ on, and let $E(s, w)$ be the environment's.
Writing $\rho(s,w)$ for the number of initial positions at which $w$ matches the
queried reversal of $s$, the dense variant gives
$E = \tfrac1k \sum_{j < \rho} \gamma^j$, and the sparse variant gives
$E = \gamma^{\rho}\,\rho/k$ on failure ($\rho < k$) and $\gamma^{k-1}$ on
completion. In both variants $E$ is strictly increasing in $\rho$ for fixed $s$,
in the sparse variant exactly when $\gamma > (k-1)/k$, the same monotonicity as
Proposition~\ref{prop:redundancy}. Call the full-success value $F$; every partial
run is worth strictly less than $F$.

\emph{Step 1 (pinning).} Fix a watch behavior $\alpha$ and let $x(s)$ be the
machine state at time $k$ under dictation $s$, a deterministic function of $s$.
Applying Definition 2 of the main text to the policy $(\alpha, w(\cdot))$ and to
the same policy with the word changed at a single dictation, then subtracting,
gives
\begin{equation*}
G\big(x(s), w\big) - G\big(x(s), w'\big) = E(s, w) - E(s, w')
\tag{$\star$}
\end{equation*}
for all $s$, $w$, and $w'$. Suppose $x(s) = x(s'')$ with $s \ne s''$. By
$(\star)$ the function $w \mapsto E(s, w) - E(s'', w)$ is constant. Evaluating it
at the all-correct words $w_s$ and $w_{s''}$ and subtracting gives
$E(s'', w_s) + E(s, w_{s''}) = 2F$. But playing the correct word for one
dictation under the other fails at their first disagreement, so both terms are
strictly below $F$, a contradiction. Hence $x(\cdot)$ is one-to-one for every
watch behavior.

\emph{Step 2 (perfection).} The same subtraction applied to continuations pins
every later state. For a machine state $\xi$ reached at a play time by some
(dictation, action-prefix) pair, $(\star)$ gives that any two pairs reaching
$\xi$ have tail-reward functions differing by a constant. A dead pair's tail is
identically $0$ while an alive pair's tail takes at least two values
(all-correct earns the full remainder, wrong-now earns $0$ or the current partial
amount), so dead and alive pairs never share a state. Now take two alive pairs at
positions $i$ and $i''$. In the sparse variant, evaluating the constant
difference at ``wrong now'' gives $(i - i'')/k$ and at ``correct once, then
wrong'' gives $\gamma\,(i{+}1 - i''{-}1)/k$, and equating the two forces
$i = i''$. In the dense variant, ``wrong now'' earns $0$ for both pairs, so the
constant is $0$ and the tails are equal outright; evaluating at ``all correct''
gives $\tfrac1k \sum_{j < k-i} \gamma^j = \tfrac1k \sum_{j < k-i''} \gamma^j$,
which again forces $i = i''$. With equal positions the constant is $0$ in both
variants, so the tail-reward functions are identical, and this forces identical
remaining targets, since playing the correct word of one pair under the other
would otherwise make the two tails differ. Every reachable play-phase cell
$(o_{\mathrm{play}}, \xi, a)$ therefore contains pairs sharing aliveness,
position, and remaining targets, and the distribution of $(o_{t+1}, r_t)$ on it
is fixed, with next observation $o_{\mathrm{play}}$ and the deterministic reward
the variant assigns. Watch-phase cells are constant as well, since the tokens of
a uniform $s$ are independent and uniform. Every cell satisfies Eq.\ (3) of the
main text, so the machine is perfect; since $\mathcal{T}$ was arbitrary, the
reward is structurally dense.
\end{proof}

\begin{remark}[Greedy behavior and the gate order]
\label{rem:tinygreedy}
W.r.t.\ the sequence machine (states are the remaining targets and the position),
the dense variant is potentially dense. The one-step reward is $\tfrac1k$ for the
correct token and $0$ otherwise at every position, so the greedy RM policy is
optimal and Theorem 1 predicts the null outcome directly. The sparse variant is
potentially sparse. At position $i \ge 1$ the one-step reward is $i/k$ for a
wrong token and $0$ for the correct one, so the greedy policy fails on purpose,
and $\phi = V^\star$ repairs it. Dense and potentially sparse would alone predict
a helpful bonus; the observed null is given by
Proposition~\ref{prop:redundancy}, whose gate is read first.
\end{remark}

\subsection{Bonus Redundancy on the Scheduled Stream}

\begin{proposition}[Bonus redundancy]
\label{prop:redundancy}
In the dictation abstraction of TinyReproduce, let
$b_t = g_t(o_0, \dots, o_t) \ge 0$ be any per-step intrinsic bonus that is a
function of the observation prefix. In the dense variant, and in the sparse
variant whenever $\gamma > (k-1)/k$, for every $\lambda \ge 0$ the optimal policy
sets of $r + \lambda b$ and of $r$ coincide, and all optimal policies of $r$ have
identical intrinsic return.
\end{proposition}

\begin{proof}
Fix the dictated sequence $s$ and a policy, and let $\tau$ be the (random) number
of correct tokens. The within-episode observation stream is a deterministic
function of $(s, \tau)$, since dictation is scheduled, the reproduction
observation is the constant $o_{\mathrm{play}}$, and the episode ends at the
first error or at completion. Moreover, for the same $s$, the stream under $\tau$
is an exact prefix of the stream under any $\tau' > \tau$, because the two runs
coincide through the step at which the token at position $\tau$ is played and the
shorter run simply ends. Since each $b_t$ is a nonnegative function of the prefix
$o_{0:t}$, the cumulative discounted bonus under $\tau'$ equals that under $\tau$
plus nonnegative terms, so the intrinsic return is nondecreasing in $\tau$ for
each fixed $s$.

The extrinsic return is strictly increasing in $\tau$ for each fixed $s$. In the
dense variant each additional correct token adds a positive discounted term. In
the sparse variant the terminal reward is $\gamma^{\,t(\tau)}\tau/k$ with
$t(\tau)$ increasing by one per additional token, and
$\gamma^{\,t(\tau)+1}(\tau{+}1)/k > \gamma^{\,t(\tau)}\tau/k$ for all
$\tau < k$ exactly when $\gamma(\tau+1) > \tau$, which holds for all
$\tau \le k-1$ iff $\gamma > (k-1)/k$.

Now compare any two policies. A policy that fails to achieve $\tau = k$ with
positive probability is, on that event, strictly worse extrinsically and weakly
worse intrinsically than the policy that corrects the corresponding tokens and
agrees elsewhere. So for every $\lambda \ge 0$ the optimal policies of
$r + \lambda b$ and of $r$ coincide, and all optimal policies of $r$ have
identical intrinsic return.

\emph{Remark.} The hypothesis $b_t = g_t(o_{0:t})$ covers E3B exactly, since its
features and covariance are computed from the observations received so far.
NovelD scores the arrival observation $o_{t+1}$, so the comparison acquires a
single boundary term at the divergence step; the conclusion then holds up to
$\lambda$ times one first-visit bonus term, consistent with the
at-most-$\pm0.08$ effects observed. The statement concerns the objective, not
optimization dynamics, as does Theorem 1.
\end{proof}

\noindent With $k = 10$, the sparse-variant hypothesis requires
$\gamma > 0.9$; the experimental discount $\gamma = 0.99$
(Table~\ref{tab:hp-tiny}) satisfies it, so the proposition covers the runs
reported in Table 3 of the main text.

\subsection{Proof of Theorem 1}

\textit{For a POMDP $\mathcal{M} = \tuple{S, A, O, p, \omega, r, \gamma, \mu}$, if $r$
is structurally dense, POMDPRM
$\mathcal{T} = \tuple{S, A, O, p, \omega, \gamma, \mu, P, L,
\mathcal{R}_{POA}}$ encodes $\mathcal{M}$ (hence is perfect, by Definition 3),
and $r$ is potentially dense w.r.t.\ $\mathcal{T}$, then the greedy RM policy
over $\mathcal{T}$ is optimal for $\mathcal{M}$ among all observation-history
policies.
}

\begin{proof}
Structural density makes $\mathcal{T}$ perfect, so by
Remark~\ref{rem:generalkernel} the kernel of Definition 4 equals the canonical
$\bar P(o' \mid o, u, a) := \Pr(o_{t+1}{=}o' \mid o_t{=}o, x_t{=}u, a_t{=}a)$,
which does not depend on the policy or on $t$ on reachable pairs. Since
$x_{t+1}$ is a deterministic function of $(x_t, o_t, a_t, o_{t+1})$, the pair
process is a finite MDP $\mathcal{Y}$ on $Y = O \times U$ with reward
$\bar R(y, a, y') := \delta_r(u, \sigma)$, discount $\gamma$, and initial
distribution induced by $(\mu, \omega, u_0)$. Let $V^\star$, $Q^\star$,
$J^\star$ be its optimal value function, action-value function, and optimal
return, which exist and satisfy the Bellman optimality equations because
$\mathcal{Y}$ is finite and $\gamma < 1$.

First, $\sup_\pi J(\pi) = J^\star$. The machine state is a deterministic
function of the observation history, so every observation-history policy induces
a policy for $\mathcal{Y}$ with the same machine return
$\tilde V(\pi) := \mathbb{E}^\pi[\sum_t \gamma^t \tilde r_t]$, and conversely;
and in a finite discounted MDP, history-dependent policies do not outperform
stationary ones, so $\sup_\pi \tilde V(\pi) = J^\star$. By Definition 2 of the
main text, $\tilde V(\pi) = J(\pi)$ for every $\pi$, and the claim follows. No
step-by-step equality between $\tilde r_t$ and $r_t$ is used or available, since
encoding is an equality of expected returns, and expected returns are all that
optimal behavior depends on.

Second, take $\phi := V^\star$, a function on $U \times O$ as Definition 5
requires. At every pair $(y, a)$,
\begin{align*}
\sum_{o'} \bar P(o' \mid y, a)\big[\bar R(y, a, y')
  + \gamma V^\star(y')\big] - V^\star(y)\\
  = Q^\star(y, a) - V^\star(y),
\end{align*}
so $\pi^{V^\star}$ is uniform over $\argmax_a Q^\star(y, \cdot)$, and any
randomization over the optimal-action set of a finite MDP is optimal, hence
$J(\pi^{V^\star}) = J^\star$.

Third, suppose the greedy RM policy $\pi$ had $J(\pi) < J^\star$. Then
$J(\pi) < J(\pi^{V^\star})$, and $\phi = V^\star$ would witness potential
sparsity w.r.t.\ $\mathcal{T}$, contradicting the hypothesis. Hence
$J(\pi) = J^\star = \sup_{\pi'} J(\pi')$, and the greedy RM policy is optimal
for $\mathcal{M}$.

\emph{What happened.} Perfection collapses the POMDP to a finite MDP over the
pairs (observation, machine state). In that MDP, shaping with the optimal value
function converts the one-step reward into the optimal advantage, so myopic
greed becomes optimal. A potentially dense reward is one for which no such
correction can help, which is only possible when myopic greed was optimal
already.
\end{proof}
\else
\fi

\end{document}